\documentclass{article}

\PassOptionsToPackage{numbers, compress}{natbib}
\usepackage[final]{neurips_2023}

\usepackage[utf8]{inputenc} % allow utf-8 input
\usepackage[T1]{fontenc}    % use 8-bit T1 fonts
\usepackage{hyperref}       % hyperlinks
\usepackage{url}            % simple URL typesetting
\usepackage{booktabs}       % professional-quality tables
\usepackage{amsfonts}       % blackboard math symbols
\usepackage{nicefrac}       % compact symbols for 1/2, etc.
\usepackage{microtype}      % microtypography
\usepackage{xcolor}         % colors
\usepackage{microtype}
\usepackage{graphicx}
\usepackage{caption}
\usepackage{subcaption}
\usepackage{amsmath}
\usepackage{dsfont}
\usepackage{wrapfig}

\usepackage{algorithm}
\usepackage{algpseudocode}

\DeclareMathOperator*{\argmax}{argmax}

\usepackage{amsthm}

\newtheorem{proposition}{Proposition}[section]

\newcommand{\reals}{\mathbb{R}}

\renewcommand{\vec}[1]{\ensuremath{\mathchoice%
{\mbox{\boldmath$\displaystyle\mathsf{#1}$}}%
{\mbox{\boldmath$\textstyle\mathsf{#1}$}}%
{\mbox{\boldmath$\scriptstyle\mathsf{#1}$}}%
{\mbox{\boldmath$\scriptscriptstyle\mathsf{#1}$}}}}

\title{Successor-Predecessor Intrinsic Exploration}

\author{%
Changmin Yu$^{1, 2}$ \quad Neil Burgess$^{1, \ast}$ \quad Maneesh Sahani$^{2, \ast}$ \quad Samuel J. Gershman$^{3, \ast}$ \\
$^1$Institute of Cognitive Neuroscience; \quad $^2$Gatsby Computational Neuroscience Unit; % \quad $^3$Centre for Artificial Intelligence; 
\\UCL, London, United Kingdom \\
$^3$Department of Psychology, Harvard University, Cambridge, United States \\
$^\ast$ Joint senior authors%\\
}
\begin{document}

\maketitle

\begin{abstract}
  Exploration is essential in reinforcement learning, particularly in
  environments where external rewards are sparse. Here we focus on exploration
  with \textit{intrinsic rewards}, where the agent transiently augments the
  external rewards with self-generated intrinsic rewards. Although the study of
  intrinsic rewards has a long history, existing methods focus on composing the
  intrinsic reward based on measures of future prospects of states, ignoring the
  information contained in the retrospective structure of transition sequences.
  Here we argue that the agent can utilise retrospective information to
  generate explorative behaviour with structure-awareness, facilitating
  efficient exploration based on global instead of local information. We propose
  \textit{Successor-Predecessor Intrinsic Exploration} (SPIE), an exploration
  algorithm based on a novel intrinsic reward combining prospective and
  retrospective information. We show that SPIE yields more efficient and
  ethologically plausible exploratory behaviour in environments with sparse rewards and bottleneck states 
  than competing methods. We also implement SPIE in deep reinforcement
  learning agents, and show that the resulting agent achieves stronger empirical
  performance than existing methods on sparse-reward Atari games. 
\end{abstract}

\section{Introduction}
\label{sec: intro}

% Reinforcement learning (RL) agents interact with the environment, and through
% trial-and-error, learn the optimal policy with respect to extrinsic rewards. In
% order to achieve global optimality, a necessary condition is for the agent to
% acquire information of all states in the state space. In practice, it is rarely
% possible for the RL agent to gather information about all valid state
% transitions, but the asymptotic performance usually improves with increased
% knowledge of the environment.  Hence, it is essential for RL agents to employ
% exploration strategies to achieve sufficient coverage of the state space
% efficiently, whilst the impact on learning of the optimal policy is minimised.

The study of exploration in reinforcement learning (RL) has produced a broad
range of methods~\citep{sutton2018reinforcement, amin2021survey}, ranging from
simple methods such as pure randomization~\citep{tesauro1995temporal,
sutton2018reinforcement, dabney2020temporally}, to more sophisticated methods
such as targeted exploration towards states with high
uncertainty~\citep{osband2016deep, pathak2017curiosity, laskin2022cic} and
implicit exploration with entropy maximization~\citep{neu2017unified,
haarnoja2018soft}. Intrinsic exploration, a highly effective class of methods,
uses intrinsic rewards based on the agent's current knowledge of the
environment, hence informing targeted exploration towards states with high
predictive uncertainty or state occupancy
diversity~\citep{schmidhuber1991curious, gregor2016variational,
pathak2017curiosity, machado2020count}. However, existing approaches define the
intrinsic reward based solely on prospective or empirical marginal information
about future states, ignoring retrospective information (e.g., does a given
state always precedes the goal state, hence should be more frequently
traversed?). We argue that the retrospective information contains useful signals
about the connectivity structure of the environment, hence could facilitate more
efficient targeted exploration.
For example, consider a clustered environment with bottleneck states connecting
the clusters (Figure~\ref{fig: grids_demo}), exploration based on local
information (e.g., visitation counts) would discourage the agent from traversing
bottleneck states, despite the key roles these states play in connecting
different clusters. Guiding the agents to visit such ``bottleneck'' states in
the face of minimal local information gain is essential in driving efficient and
biologically plausible exploration. Here we study the contribution of
retrospective information for global exploration with intrinsic motivation.

One of the most successful recent intrinsic exploration
algorithms~\citep{machado2020count} uses the successor representation
(SR;~\citep{dayan1993improving,gershman18}) to generate intrinsic rewards. The SR
represents each state in terms of successor states. The row norms of the SR can be used as an
intrinsic reward that generalises count-based exploration
\citep{machado2020count}. As we discuss in Section~\ref{sec: method}, the SR
contains not only prospective information, but also retrospective
information about expected predecessors. This information can be utilised to
construct a novel intrinsic reward which overcomes some of the problems
associated with purely prospective intrinsic rewards, such as untargeted exploration, the augmented reward function is non-stationary, and asymptotic uniformity. 

We provide a brief overview of background and relevant literature in
Section~\ref{sec: prelim}, and formally introduce the novel intrinsic
exploration method, \textit{Successor-Predecessor Intrinsic Exploration} (SPIE),
in Section~\ref{sec: method}. We propose two instantiations of SPIE for
discrete and continuous state spaces, with comprehensive empirical
examinations of properties of SPIE in discrete state space. We show that SPIE
facilitates more efficient exploration, in terms of improved sample efficiency
of learning and higher asymptotic return, through empirical evaluations on both
discrete and continuous environments in Section~\ref{sec: experiments}.

% We consider the problem of intrinsic exploration in reinforcement learning,
% specifically in sparse-reward / hard-exploration problems. Existing methods
% predominantly exploit the usage of the acquired knowledge of the agent as the
% intrinsic motivation for exploration, such that state-action transitions are
% rewarded based on the perceived novelty of the agent. Here we introduce a novel
% intrinsic exploration strategy, by implicitly exploiting the connectivity
% structure of the state space, to provide an intrinsic motivation for transitions
% towards ``bottleneck'' states. Intuitively, the intrinsic bias provides an
% incentive for repeated traversal of the ``bottleneck'' states, in the absence of
% novelty, that lead to regions of the state space that is potentially
% underexplored, hence facilitating more efficient coverage of the state space,
% both within and between clutters.

\section{Background and related work}
\label{sec: prelim}

\textbf{Reinforcement Learning Preliminaries}. We consider the standard RL
problem in Markov Decision Processes (MDP), defined by the tuple, $\langle
\mathcal{S}, \mathcal{A}, \mathcal{P}, \mathcal{P}_{0}, \mathcal{R}, \gamma
\rangle$, where $\mathcal{S}$ is the state space, $\mathcal{A}$ is the action
space, $\mathcal{P}: \mathcal{S}\times \mathcal{A} \rightarrow
\Delta(\mathcal{S})$ is the state transition distribution (where
$\Delta(\mathcal{S}$) is the probability simplex over $\mathcal{S}$),
$\mathcal{P}_{0} \in \Delta(\mathcal{S})$ is the initial state
distribution, $\mathcal{R}: \mathcal{S}\times\mathcal{A}\rightarrow \reals$ is
the reward function, and $\gamma \in (0, 1)$ is the discount factor. 
The goal for an RL agent is to learn the optimal policy that maximises value
(expected cumulative discounted reward): $\pi^{\ast}(a|s) = \argmax_{\pi}
q^\pi(s,a), \forall (s, a)\in\mathcal{S}\times\mathcal{A}$, where $\pi: \mathcal{S}\rightarrow \Delta(\mathcal{A})$ is the
policy, and $q^\pi(s,a)$ is the state-action value function:
\begin{align}
\textstyle
    q^\pi(s,a) = \mathbb{E}_{\mathcal{P}^{\pi}}\left[\sum_{\tau=0}^{\infty}\gamma^{\tau}\mathcal{R}(s_{\tau}, a_{\tau})|s_{0}=s, a_0 = a \right] = \mathbb{E}_{\mathcal{P}^{\pi}}[\mathcal{R}(s,a) + \gamma q^\pi(s',a')]\,,
\end{align}
where $\mathcal{P}^{\pi}(s'|s) = \sum_{a}\pi(a|s)\mathcal{P}(s'|s, a)$ is the
marginal state transition distribution given $\pi$\footnote{Note that unless
otherwise stated, we assume deterministic MDP, i.e., $\mathcal{P}(s'|s, a)$ is a
delta function.}. The second equality is the recursive form of the action value
function known as the \emph{Bellman equation} \citep{bellman1966dynamic}, which
underlies temporal difference
learning~\citep{sutton2018reinforcement}:
\begin{equation}
    \hat{q}^{\pi}(s_{t}, a_{t}) \leftarrow \hat{q}^{\pi}(s_{t}, a_{t}) + \alpha \delta_{t}, \quad \delta_{t} = 
    r_{t} + \gamma \hat{q}^{\pi}(s_{t+1}, a_{t+1}) - \hat{q}^{\pi}(s_{t}, a_{t})\,,
    \label{eq: td_q}
\end{equation}
where $\hat{q}^{\pi}(s_{t}, a_{t})$ is the current estimate of the action values (with respect to $\pi$), $\delta_{t}$ is the (one-step) temporal difference (TD) error. We will
study the effect of different intrinsic rewards on the performance of online TD
learning (SARSA) in discrete state MDPs.

\textbf{The successor representation}. The SR is defined as the expected
cumulative discounted future state occupancy under the policy\footnote{For
notational simplicity, we keep the policy dependence implicit. Similar
notational choice holds for all quantities discussed in the rest of the paper
($\vec F$ and $\vec N$).}:
\begin{equation}
\textstyle
    \vec M[s, s'] = \mathbb{E}_{\mathcal{P}^{\pi}}\left[
        \sum_{\tau=0}^{\infty}\gamma^{\tau}\mathds{1}(s_{\tau}, s')|s_{0}=s
        \right] 
        = \mathbb{E}_{\mathcal{P}^{\pi}}\left[
            \mathds{1}(s_{0}, s') + \gamma \vec M(s_{1}, s') | s_{0}=s
        \right].
    \label{eq: sr_def}
\end{equation}
Given the recursive formulation, it is possible to learn the SR matrix online
with TD learning. Given the transition tuple,
$(s_{t}, a_{t}, r_{t}, s_{t+1}, a_{t+1})$, the update is
\begin{equation}
    \hat{\vec M}[s_{t}, s'] \leftarrow \hat{\vec M}[s_{t}, s'] + 
    \alpha\delta^{\vec M}_{t}, \quad 
    \delta^{\vec M}_{t} = \mathds{1}(s_{t}, s') 
    + \gamma \hat{\vec M}[s_{t+1}, s'] - \hat{\vec M}[s_{t}, s']\,,
\end{equation}
Note that these equations are analogous to TD learning for value function
estimation, except that in this case the function being learned is a
vector-valued (one-hot) representation of future states. 

% We note that the SR is insensitive to the reward structure of the environment,
% hence facilitating one-step policy evaluation given reward changes. Namely, the
% value function admits the following decomposition.
% \begin{equation}
%     \mathbf{r}^T\vec M[s, :] = \mathbf{r}^T\mathbb{E}_{\pi}\left[
%         \sum_{\tau=0}^{\infty}\gamma^{\tau}\mathds{1}(s, s_{\tau})|s_{0}=s
%         \right] = 
%         \mathbb{E}_{\pi}\left[
%             \sum_{\tau=0}^{\infty}\gamma^{\tau}\mathcal{R}(s_{\tau})|s_{0}=s
%             \right] = 
%             v^{\pi}(s)
% \end{equation}
% where $\mathbf{r} = [\mathcal{R}(s)]_{s\in\mathcal{S}}$. We see that changes in
% the reward structure can be directly reflected in the updated value function
% given the composition with the unchanged SR matrix, hence leading to
% generalization across different reward configurations. 

\textbf{First-occupancy representation}. The SR captures the expected
cumulative discounted state occupancy over all future steps. However, in many
real-world and simulated tasks, it may be preferable to reach the goal
state as quickly as possible instead of as frequently as possible. In this spirit, \citet{moskovitz2021first}
introduced the \textit{First-occupancy Representation} (FR). Formally,
the FR matrix in a discrete MDP is defined by
\begin{equation}
    \begin{split}
    \vec F[s, s'] &= \mathbb{E}_{\mathcal{P}^{\pi}}\left[
        \sum_{\tau=0}^{\infty}\gamma^{\tau}\mathds{1}(s_{\tau}=s', s'\notin\{s_{0:\tau}\})|s_{0}=s
    \right] \\
    &= 
    \mathbb{E}_{\mathcal{P}^{\pi}}\left[
        \mathds{1}(s_{t}, s') + \gamma(1-\mathds{1}(s_{t}, s'))\vec F[s_{t+1}, s']|s_{t}=s
    \right]\,,
    \end{split}
\end{equation}
where $\{s_{0:\tau}\} = \{s_{0}, s_{1}, \dots, s_{\tau-1}\}$. The
recursive formulation implies that there is an efficient TD learning rule for online
learning of the FR matrix. Given the transition tuple $(s_{t}, a_{t}, r_{t},
s_{t+1}, a_{t+1})$, the update rule is
\begin{equation}
    \hat{\vec F}[s_{t}, s'] \leftarrow \hat{\vec F}[s_{t}, s'] + \alpha\delta^{\vec F}_{t}, \quad 
    \delta^{\vec F}_{t} = \mathds{1}(s_{t}, s') + 
    \gamma(1-\mathds{1}(s_{t}, s'))\hat{\vec F}[s_{t+1}, s'] - 
    \hat{\vec F}[s_{t}, s']\,,
\end{equation}

\textbf{Intrinsic exploration in RL}. 
% Many exploration strategies have been proposed,
% ranging from simple approaches where the agent samples actions according to a broad  
% distribution (e.g., uniform distribution, Boltzmann distribution, etc.) at
% randomly sampled (or every) timesteps~\citep{sutton2018reinforcement,
% wiering1999explorations, dabney2020temporally} or explicit perturbation to the
% policy-parametrised action distribution~\citep{lillicrap2015continuous,
% fortunato2017noisy, haarnoja2018soft}, to more complex methods such as
% intrinsically motivated exploration~\citep{schmidhuber1991curious,
% auer2006logarithmic, pathak2017curiosity, burda2018exploration,
% machado2020count, zhang2021noveld,badia2020never,henaff2022exploration} and deliberate exploration with gradient-based meta
% learning~\citep{finn2017model, frans2017meta}.
Here we focus on exploration with
intrinsic motivation, where the agent augments the external rewards with
self-constructed intrinsic rewards based on its current knowledge of the
environment.
\begin{equation}
    r^{\text{tot}}(s, a) = r^{\text{ext}}(s, a) + \beta r^{\text{int}}(s, a)\,,
\end{equation}
where $r^{\text{ext}}(s,a)$ denotes the extrinsic environmental reward,
$r^{\text{int}}(s,a)$ denotes the (possibly non-stationary) intrinsic reward,
and $\beta$ is a multiplicative scaling factor controlling the relative balance
of $r^{\text{ext}}(s,a)$ and $r^{\text{int}}(s,a)$. The intrinsic reward often
operates by motivating the agent to move into under-explored parts of the state
space in the absence of extrinsic reinforcement. Many types of intrinsic rewards have
been proposed, including functions of state visitation
counts~\citep{dayan1996exploration, sutton1990integrated, fu2017ex2}, predictive
uncertainty of value estimation~\citep{osband2016deep}, and predictive error of
forward models~\citep{schmidhuber1991curious, pathak2017curiosity,
burda2018exploration}. In a closely related work, \citet{zhang2021noveld} proposes NovelD, which constructs the episode-specific non-negative intrinsic reward based on the difference between the novelty measures of temporally adjacent states along a trajectory. However in contrast to SPIE (discussed later), the key difference is that NovelD does not explicitly utilise the retrospective information for exploration and the associated intrinsic reward is episode-dependent.

\section{Successor-Predecessor Intrinsic Exploration}
\label{sec: method}

Existing intrinsic exploration methods construct intrinsic rewards based on
either the predictive information in a temporally forward fashion (e.g.,
predictive error), or the empirical marginal distribution (e.g., count-based
exploration). Here we argue that the
retrospective information inherent in experienced trajectories, though
having been largely overlooked in the literature, could also be utilised as a
useful exploratory signal. Specifically, consider the environment
in Figure~\ref{fig: grids_demo}
(\textit{Cluster-simple}), where the discrete grid world is separated into two
clusters connected by a ``bottleneck'' state. Whenever the starting and reward locations are in different clusters, the bottleneck state, $s_{\ast}$, always precedes
the goal state, regardless of the trajectory taken. Hence, the
frequent predecessor state (e.g., $s_{\ast}$), to the goal state should be traversed despite the fact that immediate information gain by
traversing the state is minimal. In the absence of extrinsic reward, if only
utilising learned prospective information based on past experience (e.g., the norm of the
online-learned SR~\citep{machado2020count}), the intrinsic motivation for exploration is merely
local hence would discourage transitions into bottleneck states. However, the
retrospective information can be utilised to identify the state transitions that
connect different sub-regions of the state space, hence incorporating the
connectivity information of the state space into guiding exploration, allowing
the agent to escape local exploration and navigate towards
bottleneck states to reach distant regions. 

We develop \textit{Successor-Predecessor Intrinsic Exploration} (SPIE) algorithm utilising intrinsic rewards based
on both prospective and retrospective information from past trajectories. Below we provide
instantiations of SPIE based on the SR for discrete and continuous state
spaces. 

\textbf{SPIE in discrete state space}. We define the \textit{SR-Relative} (SR-R)
intrinsic reward, which is defined as the SR of the future state from the
current state minus the sum of the SRs of the future state from all states. Formally,
given a transition tuple, $(s, a, r, s', a')$, we define the SR-R intrinsic
reward as:
\begin{equation}
    r_{\text{SR-R}}(s, a) = \hat{\vec M}[s, s'] - ||\hat{\vec M}[:, s']||_{1} = 
    -\sum_{\tilde{s}\in\mathcal{S}, \tilde{s}\neq s} \hat{\vec M}[\tilde{s}, s']\,,
    \label{eq: sr-r}
\end{equation}
The above equation holds in deterministic MDPs (i.e., when $s'$ is a
function of $(s, a)$). We note that the $j$-th column of the SR matrix
represents the expected discounted occupancy to state $j$, starting from every
state, hence constituting a \textit{temporally backward} measure of the
accessibility of state $j$~\citep{bailey2022predecessor}. Therefore,
$r_{\text{SR-R}}(s, a)$ consists of both a prospective measure ($\hat{\vec M}[s,
s']$) and a retrospective measure ($||\hat{\vec M}[:, s']||_{1}$), and exploring
with $r_{\text{SR}}$ is an instantiation of SPIE in discrete MDPs. Intuitively,
$r_{\text{SR-R}}(s)$ can be interpreted as penalising transitions leading into
states $s'$ that are frequently reached from many states other than $s$,
hence providing an intrinsic motivation for guiding the agent towards states
that are harder to reach in general, e.g, boundary states and bottleneck
states. We thoroughly investigate the individual contribution of prospective and retrospective information through ablation studies in Appendix~\ref{sec: sarsa_srr_ablation}, and we observe that prospective information alone does not yield optimal exploration performance, whereas utilising only the retrospective information does not degrade exploration efficiency, indicating the importance of global topological information contained in the retrospective information for intrinsic exploration. 
% We term such exploratory behaviour as ``bottleneck-seeking''
% exploration. 

In a closely related work, \citet{machado2020count} showed that
$r_{\text{SR}}(s) = 1/||\hat{\vec M}[s, :]||_{1}$ can be used as an intrinsic
reward that facilitates exploration in sparse reward settings. They
additionally showed that the row norm of the online-learned SR matrix implicitly approximates the state visitation counts, so the resulting
behaviour resembles count-based exploration.
However, a key issue associated with $r_{\text{SR}}$ is that the asymptotic
exploratory behaviour is uniformly random across all states, i.e., ${||\vec M[s,:]||_{1}\rightarrow 1/(1-\gamma), \forall s\in\mathcal{S}}$. We note that
exploration involves learning of both the environmental transition structure
$\mathcal{P}^{\pi}$ and the reward structure $\mathcal{R}$. Hence, were the SR
matrix to be known \textit{a priori} (hence $\mathcal{P}^{\pi}$ could be 
implicitly derived), no intrinsic motivation would be
introduced at any state and the resulting agent regresses back to random
exploration, omitting further efficient exploration for
learning $\mathcal{R}$. Since $r_{\text{SR-R}}$ contains the sum of columns of
the SR matrix, the asymptotic uniformity in $r_{\text{SR}}$ no longer holds,
yielding non-trivial intrinsic exploration even when the SR matrix is known and
fixed \textit{a priori}, allowing continual exploration for learning the reward
structure despite sparse extrinsic reinforcement.

\textbf{Analysis of $r_{\text{SR-R}}$ with pure exploration in grid-worlds}. We
examine exploration based on $r_{\text{SR-R}}(s)$ in
discrete grid-worlds with different topologies
(Fig.~\ref{fig: grids_demo}). We first consider pure
exploration in the absence of extrinsic reward, and evaluate the exploratory
behaviours of $4$ RL agents with different intrinsic rewards, in terms of their state
coverage. The agents we consider are: vanilla
SARSA~\citep{sutton2018reinforcement}; SARSA with $r_{\text{SR}}$
(SARSA-SR;~\citep{machado2020count}); SARSA with $r_{\text{FR}}(s) = ||F[s,
:]||_{1}$ (SARSA-FR;~\citep{moskovitz2021first}); and SARSA with
$r_{\text{SR-R}}$ (SARSA-SRR); the pseudocode for SARSA-SRR can be found in
Appendix). We consider $4$ different grid-world environments with different
configurations (Figure~\ref{fig: grids_demo}), namely, $10\times 10$ open-field grid
(\textit{OF-small}); $10\times10$ grid with two rooms (\textit{Cluster-simple}); $10\times10$ grid
with $4$ rooms (\textit{Cluster-hard}); and $20\times20$ open-field grid (\textit{OF-large}). 

Exploration efficiency is quantified as the number of timesteps taken to cover $50\%$, $90\%$ and $99\%$ of the state space. The
value estimates for all states are initialised to be $0$. Due to the
absence of extrinsic reward, the vanilla SARSA agent is equivalent to a random
walk policy, which acts as a natural baseline. We observe from Figure~\ref{fig:
grids_coverage} that SARSA-SRR yields the fastest coverage of the state space
amongst all considered agents. The SARSA-FR agent yields similar state coverage
efficiency as SARSA. SARSA-SR performs poorly in all 4 grid-worlds,
failing to achieve $50\%$ state coverage within $8000$ timesteps in all
environments other than the simplest one (\textit{OF-small}). 
% The poor exploration
% efficiency of SARSA-SR conforms with our previous discussion that due to the
% equivalence to count-based exploration, the necessity of reaching sufficient
% coverage of all nearby states before progressing to new parts of the state space
% greatly hinders efficient exploration. 
Moreover, we observe that SARSA-SRR performed consistently across the $4$ considered grid
configurations, whereas all other agents experienced
significant degradation in exploration efficiency as the size and complexity of
the environments increase. 

\begin{figure*}[t!]
    \centering
    \begin{subfigure}[b]{\linewidth}
        \centering
        \includegraphics[width=.95\linewidth]{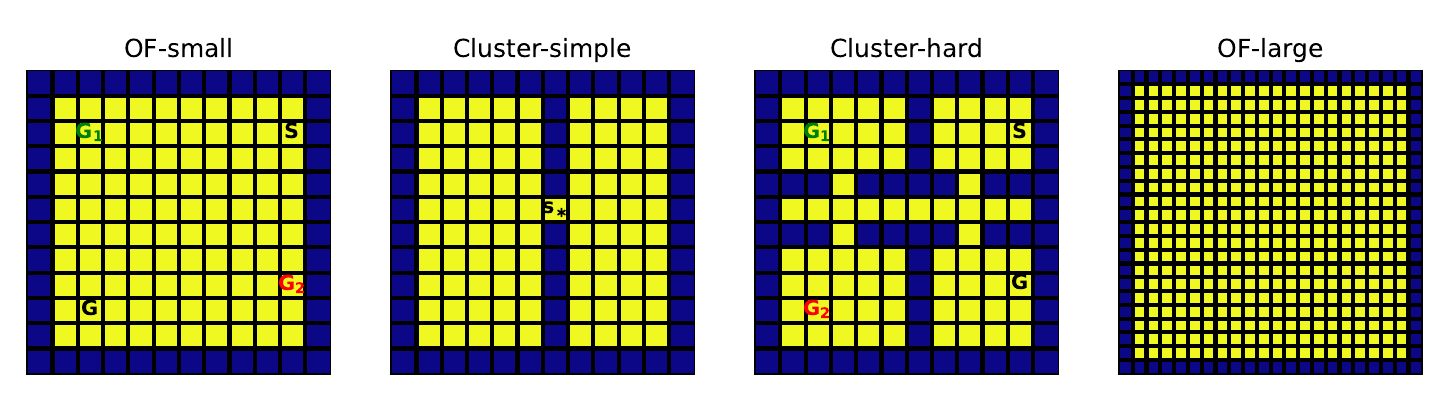}
        % \vspace{-40pt}
        % \vspace{-5pt}
        \caption{}
        \label{fig: grids_demo} 
        % \vspace{-20pt}
    \end{subfigure}
    \begin{subfigure}[b]{\linewidth}
        \centering
        \includegraphics[width=\linewidth]{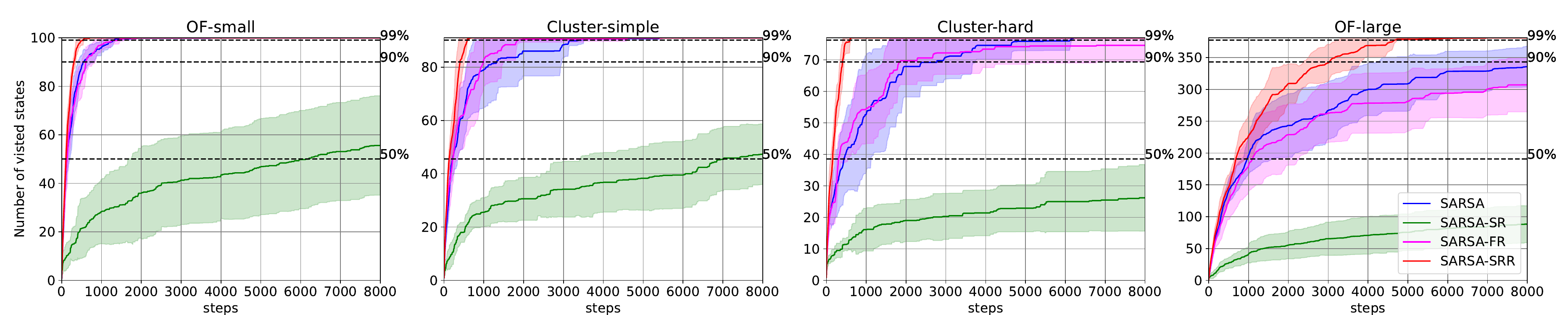}
        % \vspace{-5pt}
        \caption{}
        \label{fig: grids_coverage}
    \end{subfigure}
    \begin{subfigure}[b]{\linewidth}
        \centering
        \includegraphics[width=\linewidth]{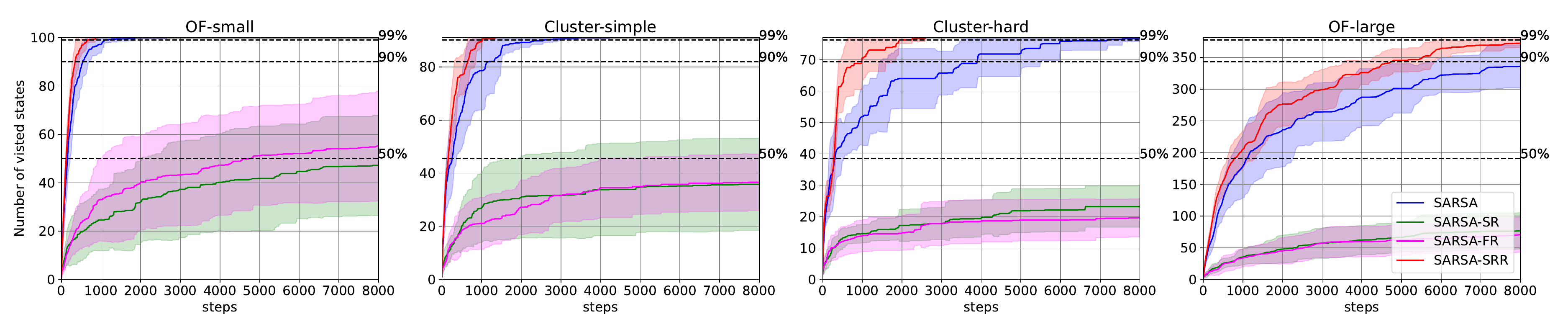}
        % \vspace{-5pt}
        \caption{}
        \label{fig: grids_coverage_full_sr}
    \end{subfigure}
    \caption{\textbf{Evaluation of exploration efficiency in grid worlds.} (a)
    Grid worlds with varying size and complexity.
    `S' and `G' in \textit{OF-small} and \textit{Cluster-hard} represents the
    start and goal states in the goal-oriented reinforcement learning task; colored $G_{1}$ and $G_{2}$ in \textit{OF-small} and \textit{Cluster-hard} represent the changed goal locations (see the non-stationary reward experiment in Section~\ref{sec: experiments}), $s_{\ast}$ in \textit{Cluster-simple} denote
    the bottleneck state. (b-c) Accumulated number of states
    visited against exploration timesteps, for all considered agents in all
    grid-worlds in with (a) online-learned SR matrix (b) and fixed SR matrix
    (c). All reported results are averaged over $10$ random seeds (shaded area
    denotes mean $\pm$ 1 standard error). Hyperparameters can be found in Appendix.}
    % \vspace{-10pt}
    \label{fig: grids_demo_coverage}
\end{figure*}

We note that in addition to improved exploration efficiency, SARSA-SRR exhibits ``cycling'' behaviour in pure
exploration in the $20\times 20$ two-cluster environment (Figure~\ref{fig:
sra_20x20_full}), spending the majority of its time
exploring in one cluster and periodically traverses the ``bottleneck''
states to explore the opposing clusters upon sufficient coverage of the current
cluster. Such ``cycling'' strategy exhibits
short-term memory of recent states and consistent long-term
planned exploration towards regions more distant in history. This is
potentially advantageous for environments with non-stationary reward structures
(\citep{thiebaux2006decision}), such as real-world
foraging, which require continual exploration for
identifying new rewards. We verify the capability of SARSA-SRR for dealing with non-stationary reward structure in Section~\ref{sec: experiments} (Figure~\ref{fig: grid_evals}).

% Given the deterministic grid-world environment, we could employ dynamic
% programming and value iteration to compute the ground-truth optimal value
% function and policy~\citep{bellman1966dynamic,sutton2018reinforcement}. By
% examining the asymptotic outputs of value iteration (Figure~\ref{fig:
% VI_varying_gamma}), we observe that the optimal policy for sufficiently large
% $\gamma$ value never traverses the bottleneck state to reach the opposing
% cluster, hence is contrary to our intuition that $r_{\text{SR-R}}$ facilitates
% ``bottleneck-seeking'' exploration. Hence, the strong exploration efficiency of
% SARSA-SRR must lie within the learning phase.

The complexity of analysing the properties of SARSA-SRR is
two-fold: the online learning of the SR matrix and the online update of the
Q-values. By assuming the SR matrix is known and fixed throughout
training,\footnote{We assume the policy the fixed SR matrix is dependent on is
the random walk policy unless otherwise stated.} we observe from
Figure~\ref{fig: grids_coverage_full_sr} that SARSA-SRR consistently outperforms
all competing methods, similar to what we observed when the SR (FR) matrix is
learned online. Additionally, we observe that the exploration efficiency for all
three intrinsic exploration agents drops when using the
intrinsic reward constructed with the fixed SR (FR), but SARSA-SRR yields
minimal decrease comparing to the significant degradation with SARSA-SR and
SARSA-FR. Hence, we have empirically confirmed that the improved exploration
efficiency does not stem solely from the online learning of the SR matrix, but
is a property of $r_{\text{SR-R}}$. Another long-standing issue with many existing intrinsic exploration methods is the non-stationary nature of the associated intrinsic bonus. By fixing the SR (FR) matrix, the associated $r_{\text{SR-R}}$ is stationary whilst still yielding high exploration efficiency, hence validating the utility of SPIE.

% Given the recursive definition of the SR, we could rewrite $r_{\text{SR-R}}$ as
% following.
% \begin{equation}
%     % \begin{split}
%         r_{\text{SR-R}}(s, a) 
%         % &= \vec M[s, s'] - \sum_{\tilde{s}\in\mathcal{S}}\vec M[\tilde{s}, s']\\
%         % &= \mathds{1}(s_{0}, s') + \gamma \mathbb{E}_{s_{1}\sim\mathcal{P}^{\pi}(\cdot|s, a)}\left[\vec M[s_{1}, s']|s_{0}=s\right] - 
%         % \sum_{\tilde{s}\in\mathcal{S}}\mathds{1}(\tilde{s}, s') + \gamma \mathbb{E}_{\tilde{s}_{1}\sim\mathcal{P}^{\pi}(\cdot|s, a)}\left[\vec M[\tilde{s}_{1}, s']|s_{0}=s\right]\\
%         = -1 - \gamma \sum_{\tilde{s}\neq s}\sum_{\tilde{s}_{1}\in\mathcal{N}(\tilde{s})}\pi(\tilde{a}|\tilde{s})M[\tilde{s}_{1}, s']\,,
%     % \end{split}
% \end{equation}
% where $\mathcal{N}(s)$ denotes the one-step neighbouring states of $s$ and
% $\tilde{a}$ is the action corresponds to the deterministic transition tuple
% $(\tilde{s}, \tilde{a}, \tilde{s}_{1})$.\cy{I am not sure how do we proceed from
% here.}

\textbf{SPIE in continuous state space with deep RL}.
In order to generalise $r_{\text{SR-R}}$ to continuous state space, we replace the SR with successor features
(SF;~\citep{barreto2017successor}).%\footnote{Note that the SF also allows a
%state-dependent formulation, see, e.g., \citet{machado2020count}}. 
\begin{equation}
\textstyle
        \vec\psi^{\pi}(s, a) = \mathbb{E}\left[\sum_{k=0}^{\infty}
        \gamma^k\vec\phi_{t+k}|s_{t}=s, a_{t}=a\right]
        = \vec\phi(s_{t+1}) + \mathbb{E}\left[
            \vec\psi^{\pi}(s_{t+1}, \pi(s_{t+1}))|s_{t}=s, a_{t}=a\right]
\end{equation}
where $\vec\phi(s, a)$ is a feature representation such that $r(s, a) =
\vec\phi(s, a)\cdot \vec w$, with weight parameter $\vec w$. The recursive
formulation for SF admits gradient-based learning of $\phi$ by minimising the
following squared TD loss.
\begin{equation}
    \delta^{\text{SF}}_{t} = \mathbb{E}\left[
        \left(\phi(s_{t}, a_{t}) + \gamma\psi(s_{t+1}, a_{t+1}) - \psi(s_{t}, a_{t})\right)^2
    \right]\,,
    \label{eq: td_sf}
\end{equation}
where the transition tuple $(s_{t}, a_{t}, s_{t+1}, a_{t+1})$ can be taken from
either online samples (SARSA-like) or sampled from offline trajectories
(Q-learning-like). We previously noted that the column of the SR matrix provides
a marginal retrospective accessibility of states, facilitating stronger
exploration. However, there is no SF-analogue of the column of the SR matrix. We
therefore construct the retrospective exploration objective with the
\textit{Predecessor Representation} (PR), which was proposed to measure how
often a given state is preceded by any other state given the expected cumulative
discounted preceding occupancy~\citep{namboodiri2021learning}. The formal
definition for the PR matrix under discrete MDP, $\vec
N\in\mathbb{R}^{|\mathcal{S}|\times|\mathcal{S}|}$, is defined as following.
\begin{equation}
    \vec N[s, s'] = \mathbb{E}_{\tilde{\mathcal{P}}^{\pi}}\left[
        \sum_{\tau=0}^{n}\gamma^{\tau}\mathds{1}(s, s_{n-\tau})|s_{n} = s'
        \right] = 
        \mathbb{E}_{\tilde{\mathcal{P}}^{\pi}}\left[
            \mathds{1}(s, s_{n}) + \gamma \vec N[s, s_{n-1}]
        \right],
    \label{eq: pr_def}
\end{equation}
where the expectation is based on $\tilde{\mathcal{P}}^{\pi}(s_{t}=s|s_{t+1}=s')
= \frac{\mathcal{P}^{\pi}(s, s')z(s)}{z(s')}$, the \textit{retrospective}
transition model, and $z(s) = \lim_{t\rightarrow
\infty}\mathbb{E}_{\mathcal{P}^{\pi}}[\mathds{1}(s_{t}=s)]$, denotes the
stationary distribution given policy $\pi$.

Utilising the recursive formulation for the PR matrix, we can again derive a
TD-learning rule. Namely, given the
transition tuple, $(s_{t}, a_{t}, r_{t}, s_{t+1}, a_{t+1})$, we have the
following update rule.
\begin{equation}
    \hat{\vec N}'[\tilde{s}, s_{t+1}] = \hat{\vec N}[\tilde{s}, s_{t+1}] + 
        \alpha \delta_{t}^{\vec N}, \quad 
    \delta_{t}^{\vec N} = \mathds{1}(s_{t+1}, \tilde{s}) + 
    \gamma\hat{\vec N}[\tilde{s}, s_{t}] - \hat{\vec N}[\tilde{s}, s_{t+1}].
\end{equation}
The SR and PR have a reciprocal relationship (proof in appendix):
\begin{equation}
    \vec N \text{diag}(\vec z) = \text{diag}(\vec z) \vec M\,, 
    \label{eq: sr_pr_reciprocal}
\end{equation} 
where $\text{diag}(\vec z) \in \mathbb{R}^{|\mathcal{S}|\times|\mathcal{S}|}$
denotes the diagonal matrix whose diagonal elements corresponds to the discrete
stationary distribution of the MDP under the current policy. 

Similar to how SF generalises SR, we propose the ``Predecessor Feature'' (PF)
that generalises PR.
\begin{equation}
\textstyle
    \vec\xi^{\pi}(s) = \mathbb{E}\left[\sum_{k=0}^{\infty}\gamma^k 
    \mu_{t-k}|s_{t+1}=s\right]= \vec\mu(s_{t+1}) + \gamma\mathbb{E}\left[\vec\xi^{\pi}(s_{t})|s_{t+1}=s, 
    a_{t}=a\right].
\end{equation}
Similarly to the SF, the recursive definition of the PF again allows a simple
expression of the TD error for gradient-based learning of the PF.
\begin{equation}
    \delta^{\text{PF}}_{t} = \mathbb{E}\left[
        \left(\phi(s_{t+1}) + \gamma \xi(s_{t}) - \xi(s_{t+1})\right)
    \right]\,,
    \label{eq: td_pf}
\end{equation}
We utilise the norms of SF and PF to replace the row sums in discrete settings for tractable approximation to $r_{\text{SR-R}}$
in continuous state spaces. We use the same feature vector, $\phi$, for computing the SF and PF. In order to ensure the SF and PF are of similar scales across the state space, we normalise $\phi(s)$ such that $||\phi(s)||_2 = 1$ for all $s$.  Contrary to how we define $r_{\text{SR-R}}$ as the
difference between the SR and the column sum of the SR in discrete
MDPs\footnote{Note we refer to discrete/continuous MDP as an MDP with
discrete/continuous state and action space.}, we find that setting the intrinsic reward as the
difference between the reciprocal of the norms of the SF and the PF yields better empirical performance. We hence
define the continuous Successor-Predecessor intrinsic reward as follows.
\begin{equation}
    r_{\text{SF-PF}} = \frac{1}{||\vec\phi(s_{t+1})||_{1}} - \frac{1}{||\vec\psi(s_{t}, a_{t})||_{1}}
\end{equation}

% \begin{figure}
\begin{wrapfigure}{r}{0.5\textwidth}
    \centering
    \includegraphics[width=\linewidth]{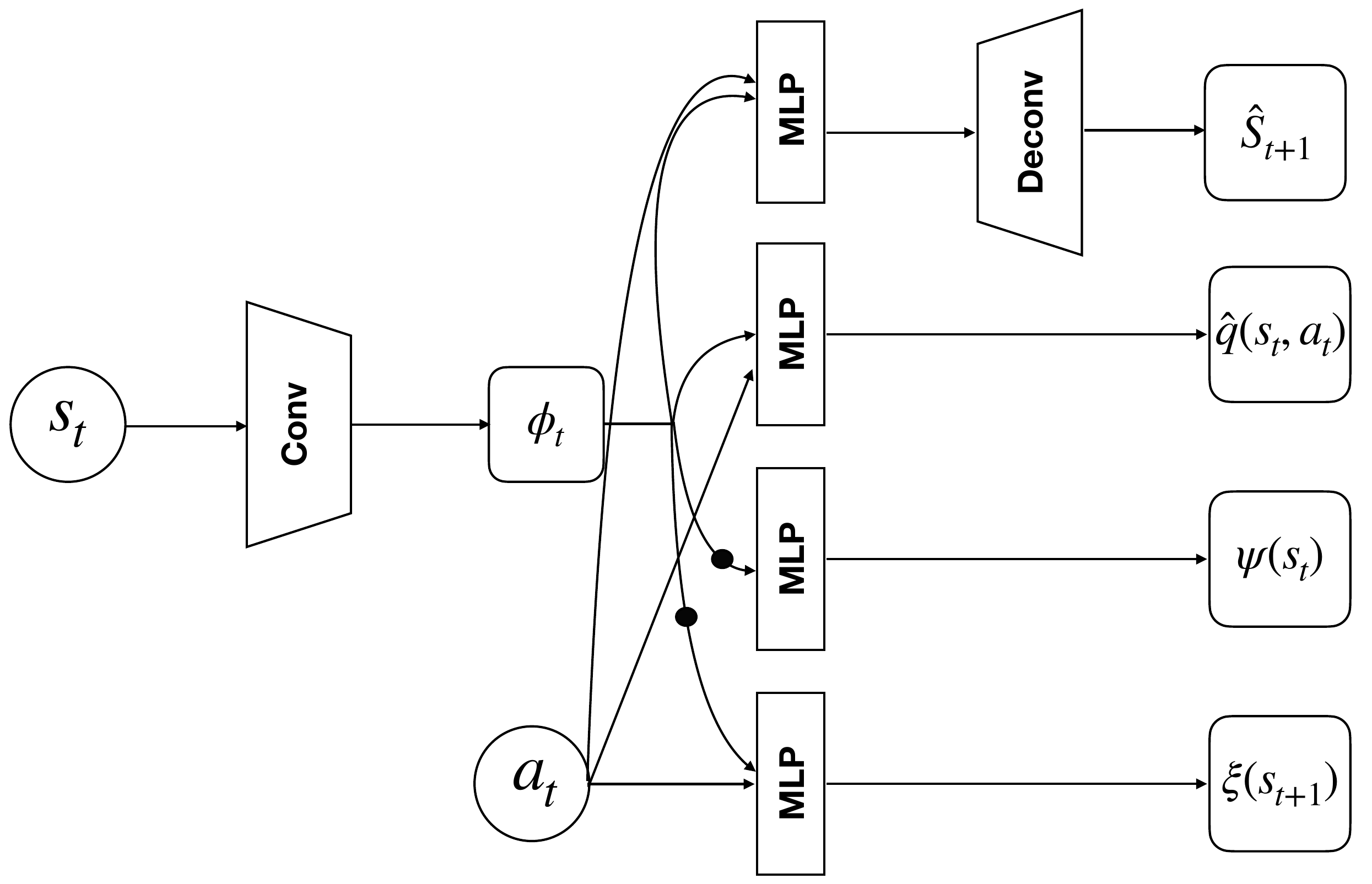}
    \caption{Graphical illustration of the neural network architecture of 
    DQN-SF-PF for Atari games. Note that the state feature vector is L2-normalised, $\phi(s) = \frac{\tilde{\phi}(s)}{||\tilde{\phi}(s)||_{2}}$, where $\tilde{\phi}(s)$ is the raw output of the convolutional encoder.}
    \label{fig: dqn_demo}
    \vspace{-15pt}
\end{wrapfigure}
\textbf{Details of deep RL implementation of $r_{\text{SF-PF}}$}. We instantiate
$r_{\text{SF-PF}}$ based on a Deep Q Network (DQN; \citep{mnih2013playing}), with
auxiliary predictive reconstruction task ($\mathcal{L}_{\text{recon}} =
\mathbb{E}\left[\left(s_{t+1}-\hat{s}_{t+1}\right)^2|s_{t}\right]$, where
$\hat{s}_{t+1}$ is the predicted next state), and separate heads for computing
the q-values, the SF, and the PF, respectively (Figure~\ref{fig: dqn_demo}). We
call this model DQN-SF-PF. Note that, following~\citet{machado2020count}, the
intermediate feature representation $\phi$ is trained given only the predictive
reconstruction and value learning supervisions, and not updated given the TD
error in the learning of the SF or the PF (the filled black circle in
Figure~\ref{fig: dqn_demo} indicating the \texttt{stop\_gradient} operation). We
adopt the same set of hyperparameters and architecture for the DQN as reported
in~\citet{oh2015action}. To make the comparison consistent, we utilise the mixed
Monte-Carlo return loss~\citep{bellemare2016unifying, machado2020count}, defined
as following.
\begin{equation}
    \begin{split}
        \mathcal{L}_{q} &= \mathbb{E}\left[
        \left((1-\tau)\delta_{\text{TD}}(s, a) + \tau \delta_{\text{MC}}(s, a)\right)^2
    \right]\,, \\
    \text{ where } \delta_{\text{MC}}(s, a) &= \sum_{t=0}^{\infty}\gamma^t\left(r(s_{t}, a_{t}) + 
    \beta r_{\text{SF-PF}}(s_{t}, a_{t}; \theta^{-})\right) - q(s, a; \theta)\,,
    \end{split}
\end{equation}
where $\delta_{\text{TD}}$ denotes the standard TD error for q-values
(Eq.~\ref{eq: td_q}), $\tau$ is the scaling factor controlling the contribution
of the TD error and the Monte Carlo error, and $\theta$ and $\theta^{-}$ denote
the parameters for the online and target DQN-SF-PF, respectively. Hence the
overall loss objective for training DQN-SF-PF is as following.
\begin{equation}
    \mathcal{L}_{\text{DQN-SF-PF}} = w_{\text{q}}\mathcal{L}_{q} + w_{\text{SF}}\delta^{\text{SF}} + 
    w_{\text{PF}}\delta^{\text{PF}} + w_{\text{recon}}\mathcal{L}_{\text{recon}}\,,
    \label{eq: loss_dqn_sf_pf}
\end{equation}
where $w_{\text{q/SF/PF/recon}}$ denotes the scaling factors for the respective
loss terms. The complete set of hyperparameters for DQN-SF-PF can be found in
the Appendix.

\section{Experiments}
\label{sec: experiments}

\textbf{Classical hard exploration tasks}. We evaluate performance of the
discrete SPIE agent (and other considered agents in Section~\ref{sec: method}: SARSA, SARSA-SR, SARSA-FR) on two classical hard exploration
tasks commonly studied in the PAC-MDP literature, \textit{RiverSwim}
and \textit{SixArms}~\citep{strehl2008analysis} (appendix Figure~\ref{fig: tabular_task}).
In both tasks, environment transition dynamics induce a bias towards states
with low rewards, leaving high rewards in states that are harder to reach. Evaluation of the agents is based on the
cumulative reward collected within $5000$ training steps. 

% \begin{table*}[t]
%     \centering
%     \caption{Evaluations on RiverSwim and SixArms (averaged over 1000 seeds,
%     numbers in the parentheses are standard deviations).}
%     % \begin{tabular}{||c c c c c c c c||} 
%     %  \hline
%     %   & SARSA & SARSA-SR & SARSA-FR & SARSA-SR-R & SARSA-SR-R(a) & SARSA-SR-R(b)
%     %   & SARSA-SR-PR\\ [0.5ex] 
%     %  \hline\hline
%     %  RiverSwim & 25,075  & 1,197,075 & 1,547,243  &
%     %  $\mathbf{2,547,156 }$ & 127,703  & 98,607 & \textbf{2,857,324} \\ 
%     %     & (1,224) & (36,999) & (34,050) & (479,655) & (530,564) & 461,385 & 419,922 \\
%     %  SixArms & 376,655  & 1,025,750  & 119,149  & $\mathbf{2,199,291
%     %  }$ & 893,530  & 791,828 & -\\
%     %  & (8,449) & (49,095) & (42,942) & (1,024,726) & (2,601,324) & 2,395,489 & -\\  
%     %  \hline
%     % \end{tabular}
%     \begin{tabular}{||c c c c c c||} 
%         \hline
%          & SARSA & SARSA-SR & SARSA-FR & SARSA-SR-R
%          & SARSA-SR-PR\\ [0.5ex] 
%         \hline\hline
%         RiverSwim & 25,075  & 1,197,075 & 1,547,243  &
%         $\mathbf{2,547,156 }$ & $\mathbf{2,857,324}$ \\ 
%            & (1,224) & (36,999) & (34,050) & (479,655) & 419,922 \\
%         SixArms & 376,655  & 1,025,750  & 119,149  & $\mathbf{2,199,291
%         }$ & $\mathbf{1,845,229}$\\
%         & (8,449) & (49,095) & (42,942) & (1,024,726) & (1,032,050)\\  
%         \hline
%        \end{tabular}
%     \label{tab: tabular_tasks}
% \end{table*}

We observe from Table~\ref{tab: tabular_tasks} that SARSA-SRR significantly
outperforms all other considered agents. Moreover, in order to further justify the utility of $\mathcal{R}_{\text{SR-R}}$ in driving exploration,
we run ablation studies by evaluating the performance of variants of SARSA-SRR (Appendix~\ref{sec: sarsa_srr_ablation}). Ablation studies reveal the importance of combining both prospective and retrospective information for exploration, as well as the benefits of dynamic balancing exploring uncertain states and bottleneck states.

In order to validate the replacement of column norm of the SR with column norm
of the PR in the construction of $r_{\text{SR-R}}$, given the reciprocal
relationship (Eq.~\ref{eq: sr_pr_reciprocal}), we empirically evaluate the
performance of the SARSA agent with the alternative intrinsic reward,
$r_{\text{SR-PR}}(s, a)=\hat{\vec M}[s, s'] - ||\hat{\vec N}[:, s']||_{1}$. SARSA-SR-PR yields comparable performance as SARSA-SRR on both
RiverSwim and SixArms (Table~\ref{tab: tabular_tasks}), empirically justifying the instantiation of SPIE
with the PR for capturing the retrospective information.

{
\begin{table}\small
\caption{Evaluations SARSA-SRR and related baseline agents on RiverSwim and SixArms (averaged over 100 seeds,
    numbers in the parentheses represents standard errors).}
  \renewcommand{\arraystretch}{1.2}
  % \hspace{-40pt}
  \centering
  \begin{tabular}{cccccc} 
    \hline
         & SARSA & SARSA-SR & SARSA-FR & SARSA-SRR
         & SARSA-SR-PR \\ [0.5ex] 
        \hline\hline
        RiverSwim & 25,075  & 1,197,075 & 1,547,243  &
        $\mathbf{2,547,156 }$ & $\mathbf{2,857,324}$ \\ 
           & (1,224) & (36,999) & (34,050) & (479,655) & (419,922)\\
        SixArms & 376,655  & 1,025,750  & 119,149  & $\mathbf{2,199,291
        }$ & $\mathbf{1,845,229}$ \\
        & (8,449) & (49,095) & (42,942) & (1,024,726) & (1,032,050)\\  
        \hline
  \end{tabular}
  \vspace{1ex}
  \label{tab: tabular_tasks}
  % \vspace{-15pt}
\end{table}
}

%\subsection{Goal-Oriented / Sparse-Reward Tasks}
%\label{sec: rl_grid}
\textbf{Goal-oriented / sparse-reward tasks}. We next evaluate the agents on
grid world tasks with a single terminal goal state (Figure~\ref{fig:
grids_demo}; \textit{OF-small} and \textit{Cluster-hard}). All non-terminal transitions yield
rewards of $-1$, and transitions into the goal state generates a reward of $0$.
Such goal-directed or sparse-reward tasks require
efficient exploration. We examine both open-field and
clustered grid-worlds. In \textit{OF-small} and
\textit{Cluster-hard} tasks, SARSA-SRR outperforms both vanilla
SARSA and SARSA-SR in terms of sample efficiency (Figure~\ref{fig: grid_evals}).
In addition, SARSA-SRR yields more stable training and performance is
more robust across different random seeds. Note that the
navigation performance of SARSA-SR during training is highly unstable, which
might attribute to its equivalence to count-based exploration given that
visitation count is only a local measure for exploration.
% From
% Figure~\ref{fig: of_grid_50_eps} and \ref{fig: cluttered_grid_50_eps}, we
% observe significant improvement over standard Sarsa and
% Sarsa-SR~\citep{machado2020count}. 
Somewhat surprisingly, the improvement for SARSA-SRR is more significant in open-field grid
world (\textit{OF-small}) rather than the clustered grid world (\textit{Cluster-hard}), in contrast to
the pure exploration experiments (Figure~\ref{fig: grids_coverage}).
Nevertheless, the improvement is strong and consistent. 

In many real-world tasks, the environment is inherently dynamic, requiring continual
exploration for adapting the optimal policy with respect to the
non-stationary task structure. One such example is random foraging, where foods are depleted upon
consumption, and new rewards appear in new locations. As argued in
Section~\ref{sec: method}, SARSA-SRR yields ``cycling'' exploratory
behaviour (Figure~\ref{fig: grid_20x20_full}), hence could facilitate continual
exploration that is potentially suitable for such non-stationary
environments. To empirically justify the hypothesis, we consider the Non-Markovian Reward Decision Process
(NMRDP;~\citep{thiebaux2006decision}), where the reward changes dynamically given
the visited state sequence. We instantiate the NMRDPs in the
grid worlds, \textit{OF-small} and \textit{Cluster-hard}, where there are three reward states
($G$, ${\color{green}G_{1}}$, ${\color{red}G_{2}}$; Figure~\ref{fig:
grids_demo}) that are sequentially activated (and deactivated) every $30$
episodes. As shown in Figure~\ref{fig: of_small_non_stationary} and \ref{fig:
cluster_hard_non_stationary},  we observe that SARSA-SRR consistently
outperforms SARSA and SARSA-SR, reaching the new goal states in increasingly
shorter timescales. This supports our idea that SPIE provides a more ethologically plausible
exploration strategy for dealing with non-stationarity. However, we note that the main focus of the current paper is on improved exploration within a single task, instead of over a stream of inter-related tasks. Here we provide preliminary evidence of potential applicability of SPIE in such continual exploration setting, and we leave more rigorous investigation in this direction for future work.

\begin{figure}[t!]
    \centering
    % \begin{subfigure}[b]{0.47\linewidth}
    %     \centering
    %     \includegraphics[width=\textwidth]{figures/of_grid_demo.pdf}
    %     \caption{}
    %     \label{fig: of_grid_demo}
    % \end{subfigure}
    % \hfill
    % \begin{subfigure}[b]{0.47\linewidth}
    %     \centering
    %     \includegraphics[width=\textwidth]{figures/cluttered_grid_demo.pdf}
    %     \caption{}
    %     \label{fig: cluttered_grid_demo}
    % \end{subfigure}\\
    \begin{subfigure}[b]{0.45\linewidth}
        \centering
        \includegraphics[width=\textwidth]{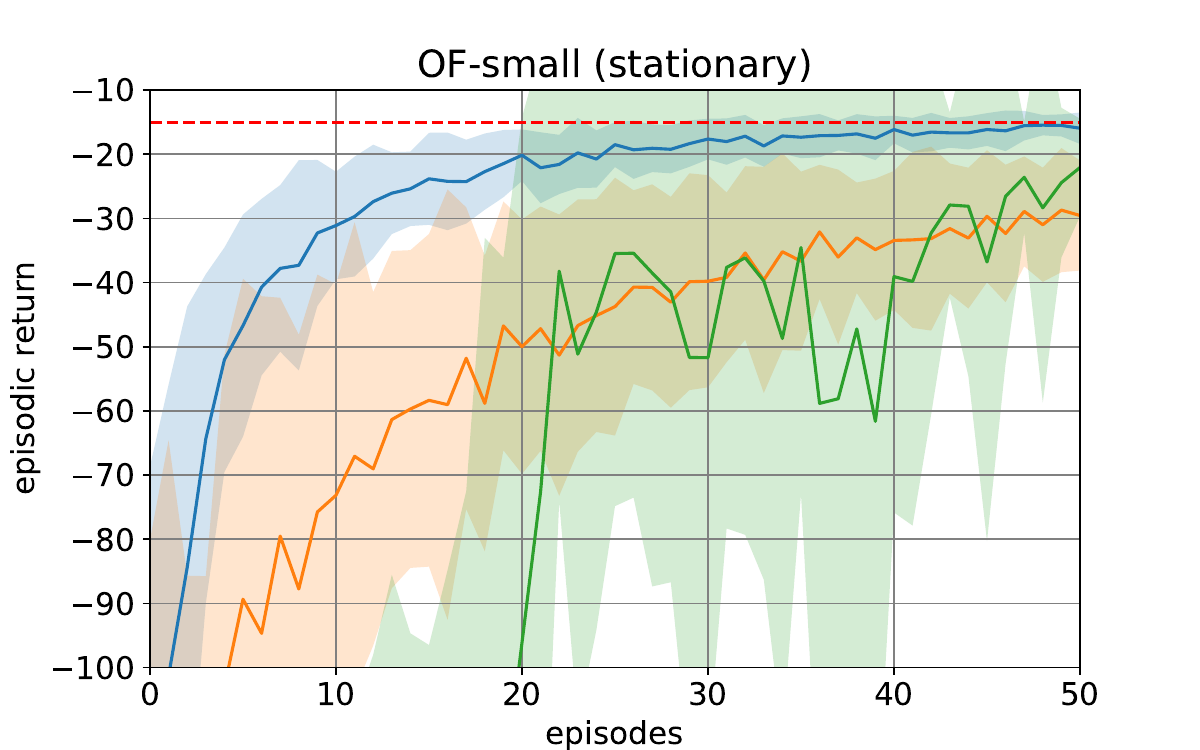}
        
        \caption{}
        
        \label{fig: of_small_stationary}
    \end{subfigure}
    % \hfill
    \begin{subfigure}[b]{0.45\linewidth}
        \centering
        \includegraphics[width=\textwidth]{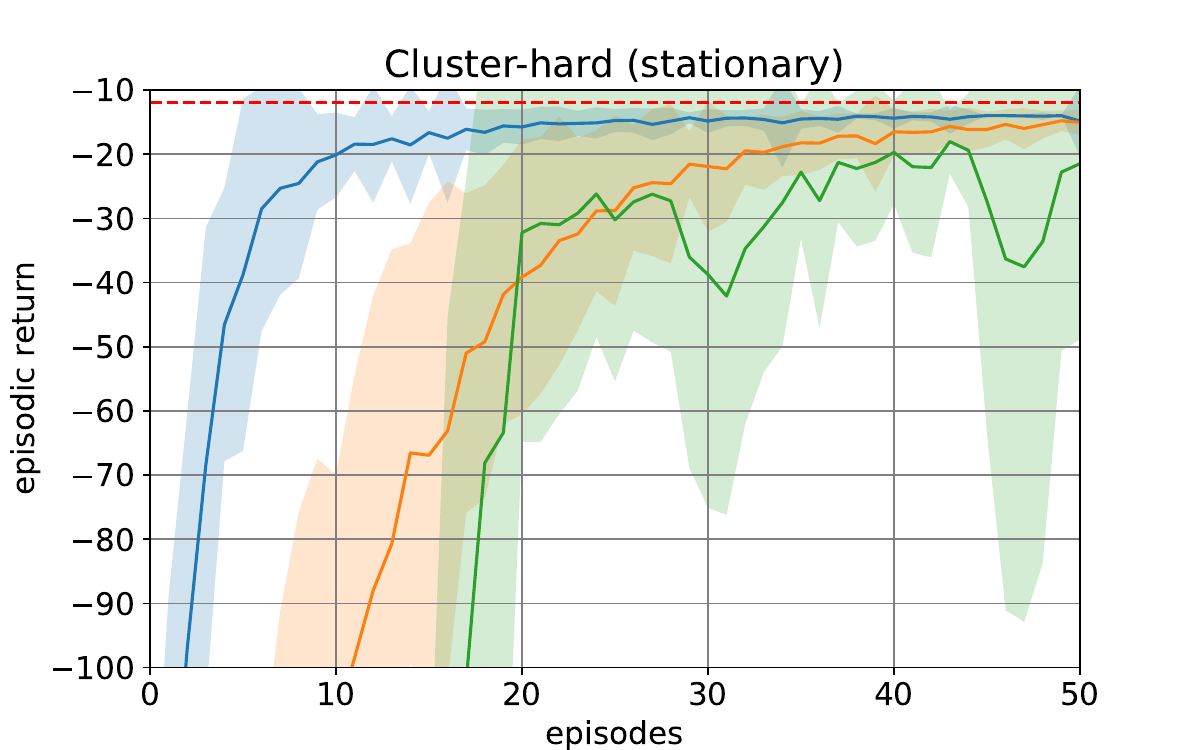}
        
        \caption{}
        
        \label{fig: cluster_hard_stationary}
    \end{subfigure}
    % \hfill
    \begin{subfigure}[b]{0.45\linewidth}
        \centering
        \includegraphics[width=\textwidth]{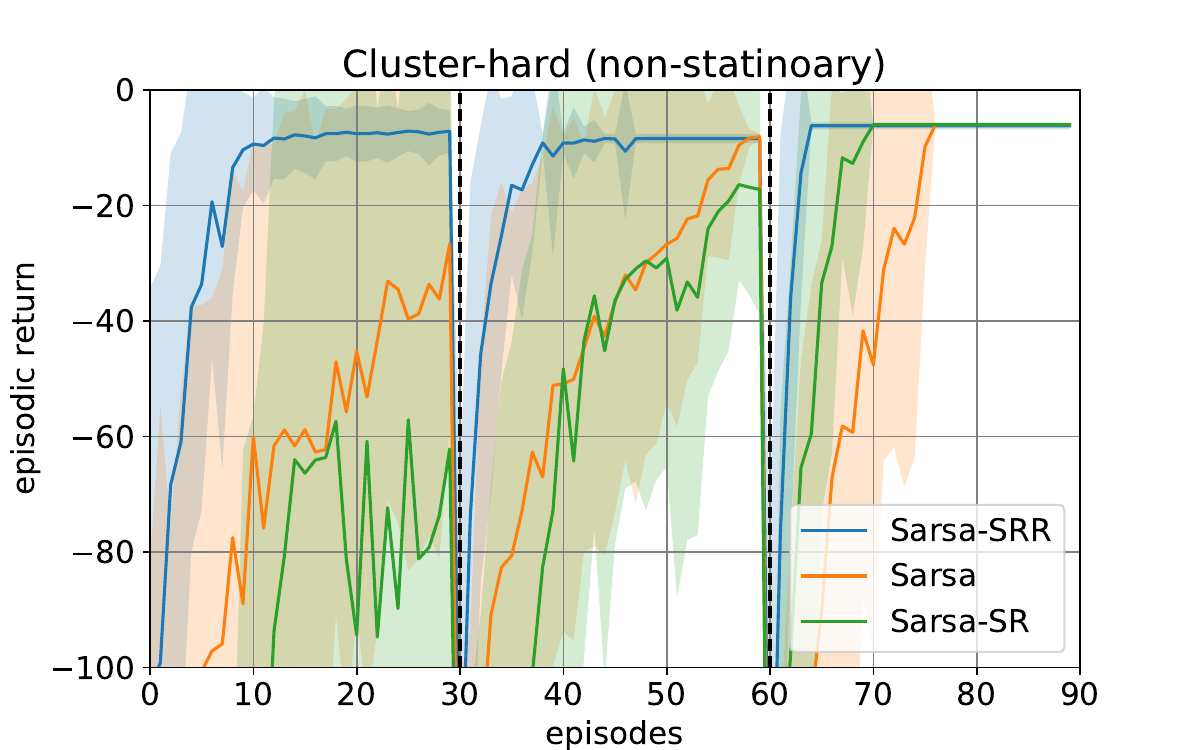}
        
        \caption{}
        \label{fig: of_small_non_stationary}
    \end{subfigure}
    % \hfill
    \begin{subfigure}[b]{0.45\linewidth}
        \centering
        \includegraphics[width=\textwidth]{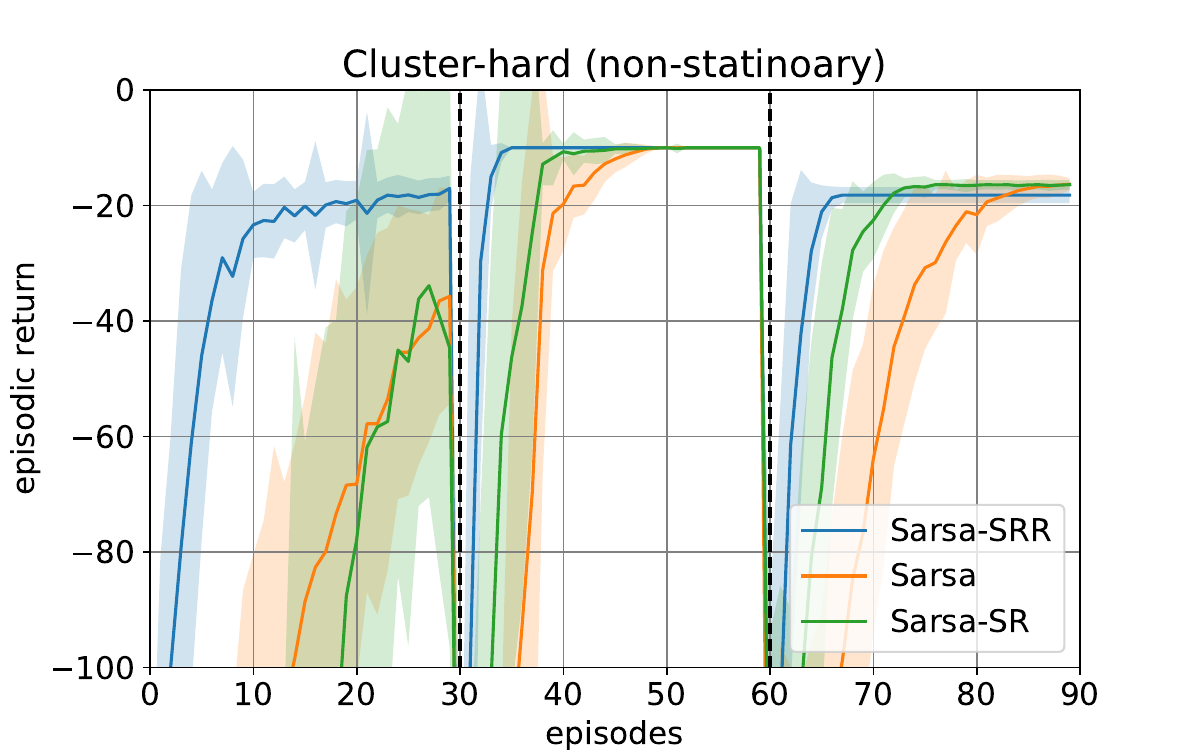}
        \caption{}
        \label{fig: cluster_hard_non_stationary}
    \end{subfigure}
    % \vspace{-5pt}
       \caption{\textbf{Goal-oriented navigation in grid worlds.} Evaluations of SARSA, SARSA-SR and SARSA-SRR on \textit{OF-small} (a) and \textit{Cluster-hard} (b) grid worlds (Figure~\ref{fig: grids_demo}) with stationary reward structure, and on \textit{OF-small} (c) and \textit{Cluster-hard} (d) with \textit{non-stationary} reward structures. The red dashed horizontal line represents the shorted path distance. The black dashed vertical lines represent the time point at which the goal change occurs.}
       \label{fig: grid_evals}
       % \vspace{-15pt}
\end{figure}

\textbf{Linear function approximation for continuous state spaces}. We next evaluate  SPIE
with function approximation. As a first step, we consider the linear features
before moving onto the deep RL setting. We consider the \textit{MountainCar}
task (Figure~\ref{fig: mountainCar_demo}; \citep{Moore90efficientmemory-based}),
with sparse reward structure, where we set the reward to $0$ for all transitions into
non-terminal states (the terminal state is indicated by the flag on the top of
the right hill). We utilise Q-learning with linear function approximation, where
we define the linear features to be the $128$-dimensional random Fourier
features (RFF~\citep{rahimi2007random}; Figure~\ref{fig: rff_demo}). The SF and
the PF are defined given the RFF, and are learned via standard TD-learning
(Eq.~\ref{eq: td_sf};~\ref{eq: td_pf}). The performance (over the
first $1000$ training episodes) of the resulting linear-Q agents with
$r_{\text{SF}}$ and $r_{\text{SF-PF}}$ is shown in Figure~\ref{fig:
mountainCar_comparison}. The agent with $r_{\text{SF-PF}}$
outperforms the opposing agent significantly, empirically justifying the
utility of SPIE in the linear function approximation regime.

\begin{figure}[t]
    \centering
    \begin{subfigure}[b]{0.3\linewidth}
        \centering
        \includegraphics[width=\textwidth]{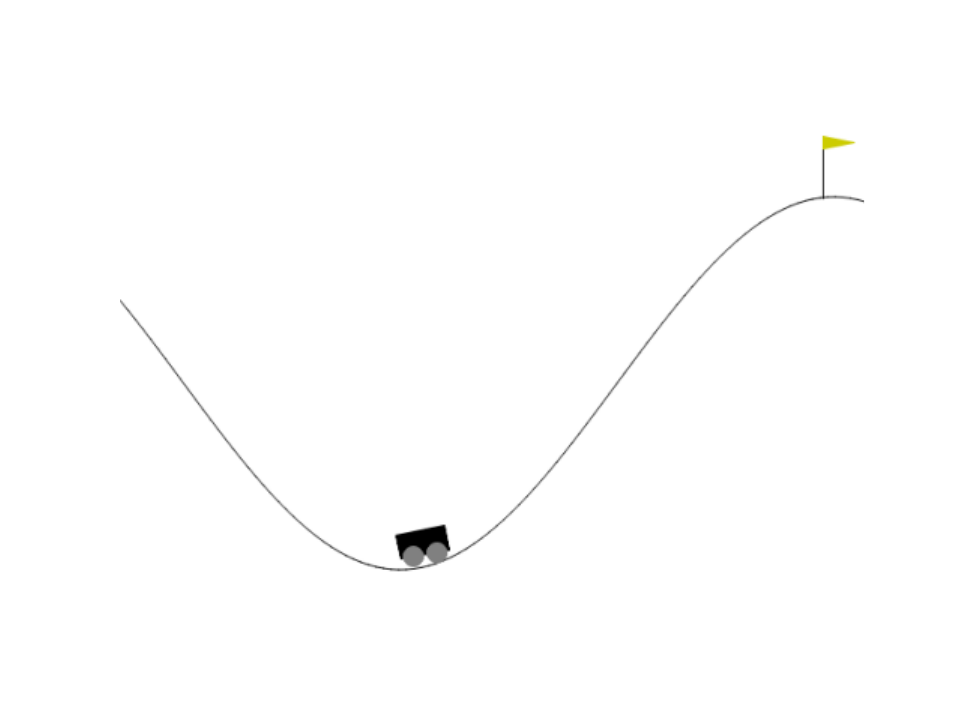}
        % \vspace{-15pt}
        \caption{}
        \label{fig: mountainCar_demo}
    \end{subfigure}
    \hfill
    \begin{subfigure}[b]{0.23\linewidth}
        \centering
        \includegraphics[width=\textwidth]{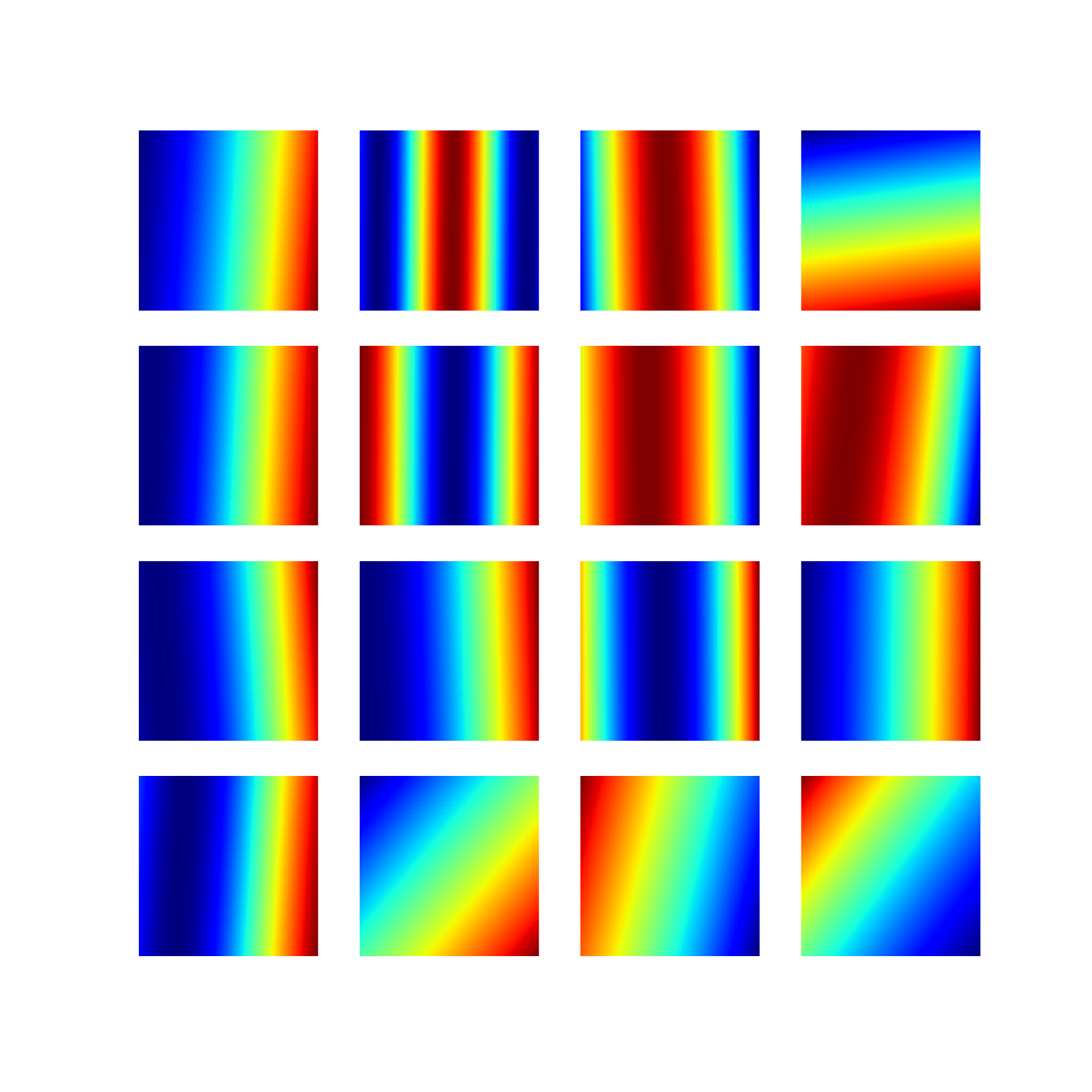}
        % \vspace{-15pt}
        \caption{}
        \label{fig: rff_demo}
    \end{subfigure}
    \hfill
    \begin{subfigure}[b]{0.37\linewidth}
        \centering
        \includegraphics[width=\textwidth]{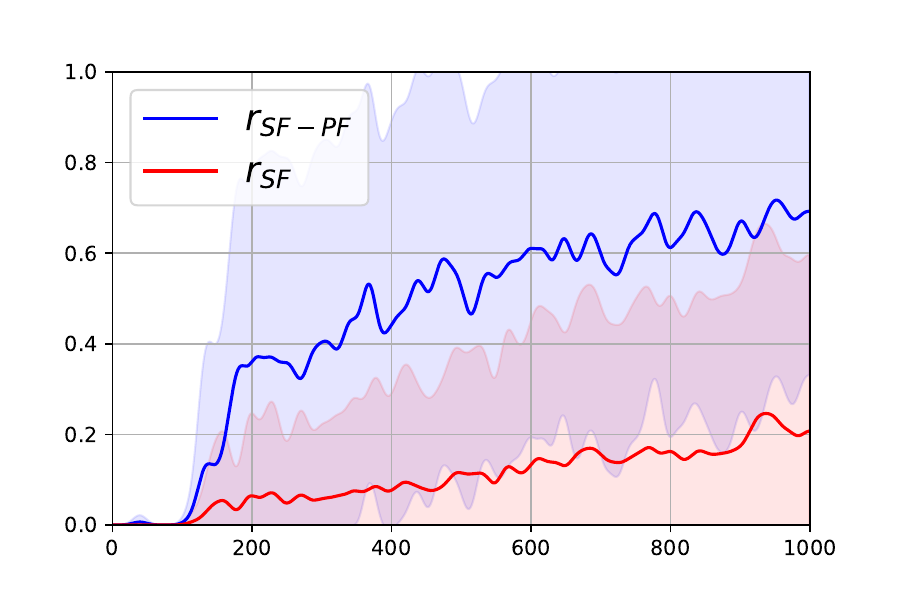}
        % \vspace{-15pt}
        \caption{}
        \label{fig: mountainCar_comparison}
    \end{subfigure}
    % \vspace{-5pt}
    \caption{\textbf{Evaluation of SPIE with linear features in MountainCar.}
    (a) Graphical demonstration of MountainCar environment; (b); Example random Fourier
    features; (c) Evaluations of Q-learning with linear function approximation
    with intrinsic rewards $r_{\text{SF}}$ and $r_{\text{SF-PF}}$ on
    MountainCar. Reported results are averaged over $10$ random seeds.}
    \label{fig: mountainCar_experiment}
    % \vspace{-20pt}
\end{figure}

%\subsection{Deep RL Evaluations}
\textbf{Deep RL instantiation of SPIE in Atari games}. We empirically evaluate $\text{DQN}_{\text{SF-PF}}$ on $6$
Atari games with sparse reward structures~\citep{bellemare2013arcade}: Freeway,
Gravitar, Montezuma's Revenge, Private Eye, Solaris, and Venture. We follow the evaluation protocol
as stated in~\citet{machado2018revisiting}, where we report the averaged
evaluation scores over $10$ random seeds given $10^{8}$ training steps.
The agent takes (stacked) raw pixel observations as inputs. Across all $4$
games, the $\beta$ values are set to $0.07$ and the discounting factor
$\gamma=0.995$. We adopt the epsilon-annealing scheme
as in~\citep{bellemare2016unifying}, which linearly decreases $\epsilon$ from $1.0$
to $0.1$ over the first $10^6$ frames. We train the network with RMSprop, with
standard hyperparameters, learning rate $0.00025$, $\epsilon = 0.001$ and decay equals $0.95$~\citep{mnih2013playing}. The discounting factors for value learning and online learning of the SF and the PF are set to $0.99$. The scaling factors in
Eq.~\ref{eq: loss_dqn_sf_pf} are set such that the different losses are on
roughly similar scales: $w_{q} = 1$, $w_{\text{SF}} = 1500$, $w_{\text{PF}} =
1500$, $w_{\text{recon}} = 0.001$. More implementation details can be found in Appendix. 

We compare DQN-SF-PF
with vanilla DQN trained with standard TD error, vanilla DQN trained with the
MMC loss ($\mathcal{L}_{q}$), Random Network Distillation
(RND;~\citep{burda2018exploration}), DQN-SR trained with the MMC
loss~\citep{machado2020count} (Table~\ref{tab: atari}). All agents are trained with the predictive
reconstruction auxiliary task. By comparing with our main baseline, DQN-SR, we
observe that DQN-SF-PF significantly outperforms DQN-SR on Four games
(Gravitar, Montezuma's Revenge, Private Eye and Solaris), whilst yielding similar performance on
the remaining two games (Freeway and Venture). Moreover, DQN-SF-PF outperforms
RND, a state-of-the-art Deep RL algorithm for exploration, on all $6$ games. 
The empirical difference is not only reflected in the asymptotic
performance, but also in the sample efficiency of learning. Specifically, for
Montezuma's Revenge, one of the hardest exploration games in the Atari suite,
our agent achieves near asymptotic performance (defined as the score given
$10^8$ training steps) with only $\sim8\times 10^6$ training frames, whereas
the performance of DQN-SR saturates at $\sim 2.4\times 10^7$ training frames (with a lower score). We emphasise that the main aim of our empirical evaluations is to validate the utility of SPIE exploration objective as a simple modification to DQN. In principle, SPIE can be integrated with any state-of-the-art RL agent, and different instantiations of SPIE could be implemented to deal with the task at hand. We leave such investigation for future work. 

{\begin{table*}[t]
    \centering
    \small
    \caption{Evaluations of SPIE with deep RL implementation on hard-exploration Atari games (averaged over 10 random seeds,
    numbers in the parentheses are  1 standard errors).}
    % \begin{tabular}{||c c c c c c c c||} 
    %  \hline
    %   & SARSA & SARSA-SR & SARSA-FR & SARSA-SR-R & SARSA-SR-R(a) & SARSA-SR-R(b)
    %   & SARSA-SR-PR\\ [0.5ex] 
    %  \hline\hline
    %  RiverSwim & 25,075  & 1,197,075 & 1,547,243  &
    %  $\mathbf{2,547,156 }$ & 127,703  & 98,607 & \textbf{2,857,324} \\ 
    %     & (1,224) & (36,999) & (34,050) & (479,655) & (530,564) & 461,385 & 419,922 \\
    %  SixArms & 376,655  & 1,025,750  & 119,149  & $\mathbf{2,199,291
    %  }$ & 893,530  & 791,828 & -\\
    %  & (8,449) & (49,095) & (42,942) & (1,024,726) & (2,601,324) & 2,395,489 & -\\  
    %  \hline
    % \end{tabular}
    \begin{tabular}{c c c c c c} 
        \hline
         & DQN & $\text{DQN}^{\text{MMC}}$ & RND & $\text{DQN}^{\text{MMC}}$-SR
         & $\text{DQN}^{\text{MMC}}$-SF-PF\\ [0.5ex] 
        \hline\hline
        Freeway & 32.4 (0.3)  & 29.5 (0.1) & 28.2 (0.2) &
       29.4 (0.1) & 27.5 (0.2) \\ 
       Gravitar & 118.5 (22.0) & 1078.3 (254.1) & 714.1 (105.9) & 457.4 (120.3) & 1223.0 (408.9)\\
        Mont. Rev. & 0.0 (0.0) & 0.0 (0.0) & 528 (314.0) & 1395.4
        (1121.8) & 1530.0 (1072.1) \\
        Private Eye & 1447.4 (2,567.9)  & 113.4 (42.3)  & 61.3 (53.7)  & 104.4
        (50.4) & 488.2 (390.9)\\
        Solaris & 783.4 (55.3) & 2132.6 (394.8) & 1395.2 (401.7) & 1890.1 (163.1) & 2455.8 (262.0) \\
        Venture & 4.4 (5.4) & 1220.1 (51.0) & 953.7 (167.3) & 1348.5 (56.5) &
        1274.0 (133.2)\\  
        \hline
       \end{tabular}
    \label{tab: atari}
    % \vspace{-10pt}
\end{table*}
}

% \vspace{-4pt}
\section{Conclusion}
% \vspace{-4pt}
The development of more efficient exploration algorithms is essential for
practical implementation of RL agents in real-world environment where sample
efficiency and optimality are vital to success. Here, we propose a general
intrinsically motivated exploration framework, SPIE, where we construct
intrinsic rewards by combining both prospective and retrospective information
contained in past trajectories. The retrospective component provides information about the connectivity structure of the environment,
facilitating more efficient targeted exploration between sub-regions of
state space given structure awareness (e.g., robust identification of
the bottleneck states; Figure~\ref{fig: grids_demo}). SPIE yields more
sample efficient exploration in discrete MDPs under complete absence of external reinforcement. Moreover, a side benefit we observe empirically is that SPIE exhibits ethologically plausible exploratory
behaviour during exploration in grid worlds (i.e., cycling between different clusters of states). 
In continuous state space, we developed a novel generalization of the
predecessor representation, the predecessor features, for capturing
retrospective information in continuous spaces. Empirical
evaluations on both discrete and continuous MDPs demonstrate that SPIE yields improvements
over existing intrinsic exploration methods, in
terms of sample efficiency of learning and asymptotic performance, and for adapting to non-stationary reward structures.  

We instantiate SPIE using the SR and the PR, but we note that SPIE is
a general framework that can be implemented with other formulations (e.g.,
predictive error in a temporally backward direction~\citep{yu2021learning,
yu2021playvirtual}) and with more advanced neural
architectures (including those currently unthought of). Although here we have examined the empirical
properties of SPIE, the theoretical underpinnings for SPIE and the bottleneck
seeking exploratory behavior bears further investigation. Specifically, more work needs to be done to probe the theoretical property of using SF and PF in continuous settings. Our definition of
$r_{\text{SR-R}}$ overlaps with the successor
contingency~\citep{namboodiri2021learning,gallistel2014temporal}, which has long
been recognised for learning causal relationship between
predictors and reward~\citep{jenkins1965judgment}. An interesting venue for
future work is to investigate the implications of SPIE for causally guided
exploration in RL. Another interesting direction for future work
is to investigate the implications of SPIE in human exploration, where we could
utilise SPIE to investigate how human balance local (e.g., visitation counts)
versus global (e.g., environment structure) information for exploration in
sequential decision tasks~\citep{acuna2008structure,brandle2022intrinsically}.

\section*{Acknowledgement}

We thank Franziska Br{\"a}ndle, James Heald, and Ted Moskovitz for useful discussions, and anonymous reviewers for valuable comments. This work is funded by the UKRI, DeepMind, the Gatsby Charitable Foundation, the Simons Foundation, the Wellcome Trust, and the Harvard Brain Initiative and by the Center for
Brains, Minds and Machines (CBMM).

% \newpage
\bibliography{example_paper}
\bibliographystyle{unsrtnat}

\newpage
\appendix

\section{More Details on Predecessor Representation}
Here we provide proofs of the reciprocal relationship between the SR and the PR.
\begin{proposition}
    $\vec N \text{diag}(\vec z) = \text{diag}(\vec z) \vec M$, where $\text{diag}(\vec z)$ is the diagonal matrix with the diagonal elements as the vector $\vec z$, and $\vec z$ is the vector of stationary distribution of $\mathcal{P}^{\pi}$ (i.e., $\vec z[i] = \lim_{t\rightarrow \infty}\mathbb{E}_{\mathcal{P}^{\pi}}[s_{t} = i]$. 
\end{proposition}

\begin{proof}
    Given the formal definition of the SR and the PR (Eq.~\ref{eq: sr_def}; \ref{eq: pr_def}), we have the following analytical expressions.
    \begin{equation}
        \vec M = (\vec I - \gamma \mathcal{P}^{\pi})^{-1}; \quad \vec N = (\vec I - \gamma \tilde{\mathcal{P}}^{\pi})^{-1};
    \end{equation}
    where $\tilde{\mathcal{P}}^{\pi}$ is the temporally reversed transition distribution. Assume matrix formulation of $\mathcal{P}^{\pi}$ and $\tilde{\mathcal{P}}^{\pi}$, $\vec P$ and $\tilde{\vec P}$ in $\mathbb{R}^{|\mathcal{S}|\times |\mathcal{S}|}$, we have the following.
    \begin{equation}
    \begin{split}
        &\tilde{\vec P}_{ij} = \mathbb{P}(s_{t} = i|s_{t+1}=j) = \frac{\mathbb{P}(s_{t+1}=j|s_{t}=i)
        \mathbb{P}(s_{t}=i)}{\mathbb{P}(s_{t+1}=j} = \frac{\vec P_{ij} \vec z_{i}}{\vec z_{j}}\,,\\
        \Rightarrow &\tilde{\vec P} \text{diag}(\vec z) = \text{diag}(\vec z)\vec P\,,
    \end{split}
    \end{equation}
    Substituting the reciprocal relationship between $\tilde{\vec P}$ and $\vec P$ into the definition of the PR, we have the following.
    \begin{equation}
    \begin{split}
        \vec N &= \left(\vec I - \gamma \text{diag}(\vec z)\vec P\text{diag}(\vec z)^{-1}\right)^{-1}\,,\\
        \vec N \text{diag}(\vec z) &= \left(\vec I - \gamma \text{diag}(\vec z)\vec P\text{diag}(\vec z)^{-1}\right)^{-1}\text{diag}(\vec z) \\
        &= \left(\text{diag}(\vec z)^{-1}(\vec I - \gamma \text{diag}(\vec z)\vec P\text{diag}(\vec z)^{-1})\right)^{-1}\\
        &= \left((\vec I - \gamma \vec P)\text{diag}(\vec z)^{-1}\right)^{-1}\\
        &= \text{diag}(\vec z)\left((\vec I - \gamma \vec P)\right)^{-1}\\
        &= \text{diag}(\vec z)\vec M
    \end{split}
    \end{equation}
\end{proof}

\section{Further results on tabular hard exploration tasks.}
\label{sec: app_tabular_further}

\subsection{Graphical illustration of tabular hard-exploration tasks.} The demos of RiverSwim and SixArms is shown in Figure~\ref{fig: tabular_task}. In both tasks, the environmental transition dynamics impose asymmetry, biasing the agent towards low-rewarding states that are easier to reach, with greater rewards available in hard-to-reach states.

\begin{figure}[b]
    \centering
    \begin{subfigure}[b]{.55\linewidth}
        \centering
        \includegraphics[width=\linewidth]{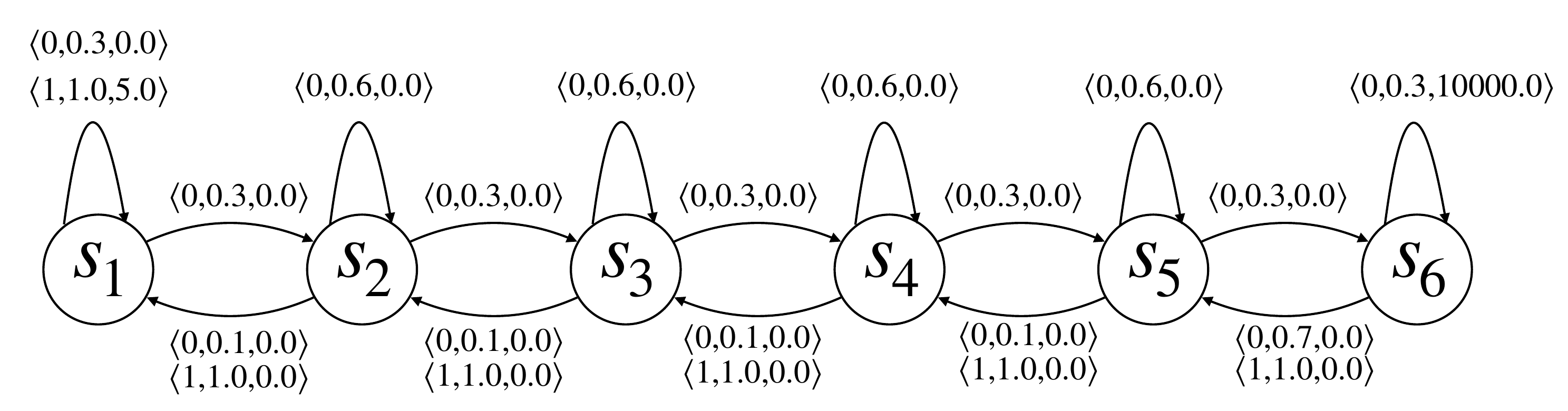}
        \caption{RiverSwim}
        \label{fig: riverSwim_demo} 
    \end{subfigure}
    \begin{subfigure}[b]{.4\linewidth}
        \centering
        \includegraphics[width=\linewidth]{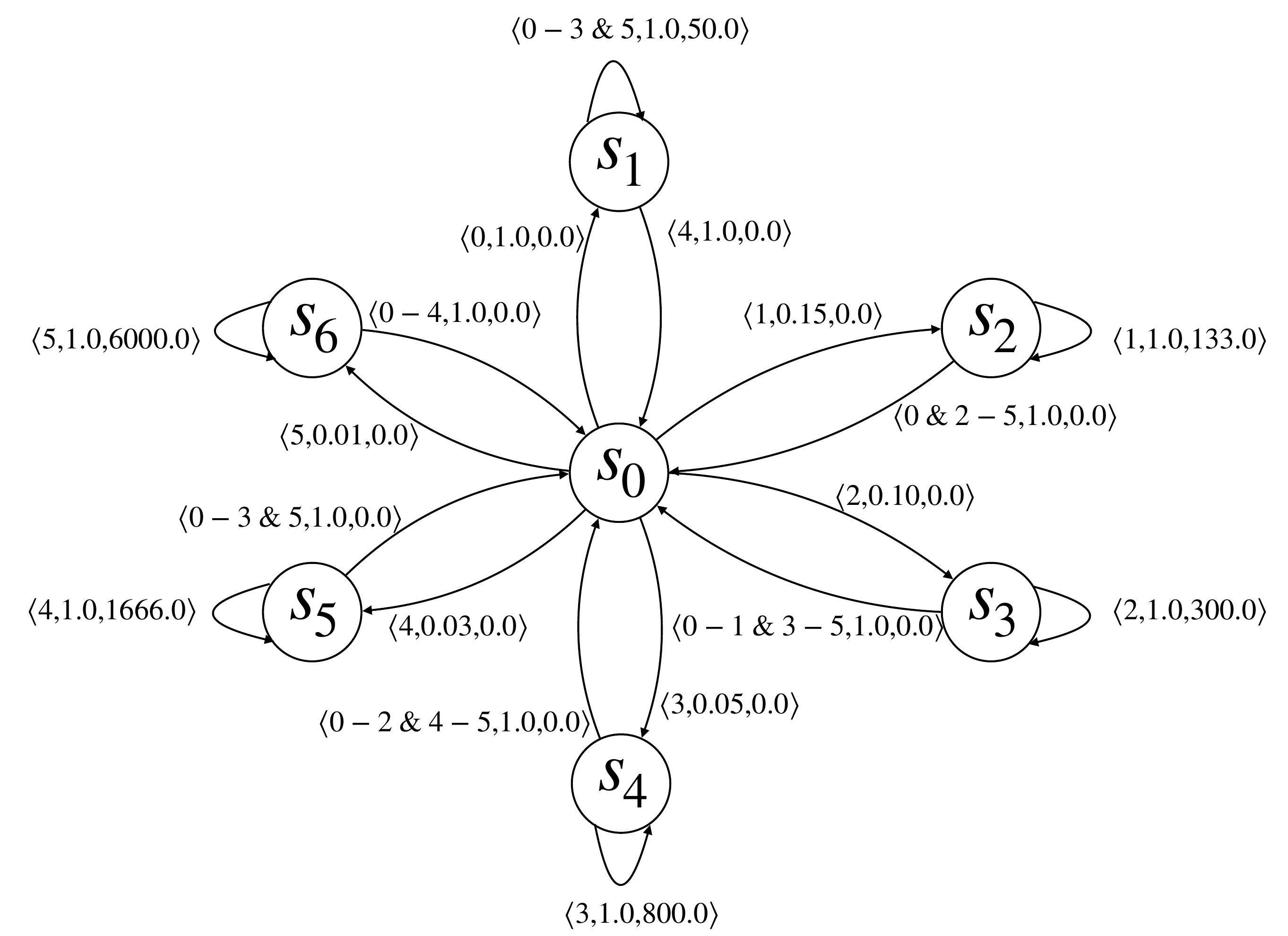}
        \caption{SixArms}
        \label{fig: sixArms_demo}
    \end{subfigure}
    % \vspace{-10pt}

    \caption{Discrete MDPs. Transition probabilities are denoted by
    $\langle\text{action}, \text{probability}, \text{reward}\rangle$. In RiverSwim (a), the
    agent starts in state 1 or 2. In SixArms (b), the agent starts in state 0.}
    \label{fig: tabular_task}

\end{figure}

\subsection{Pseudocode for SARSA-SRR.} We provide the pseudocode for SARSA-SRR in Algorithm~\ref{alg: sarsa_srr}. We note that SARSA, SARSA-SR and SARSA-FR utilise the similar algorithm, but only replacing the intrinsic bonus.

\begin{algorithm}
\caption{Pseudocode for SARSA-SRR}\label{alg: sarsa_srr}
\begin{algorithmic}
\Require $\alpha$, $\eta$, $\gamma$, $\gamma_{\text{SR}}$, $\beta$, $\epsilon$
\State $s = env.reset()$;
\State $\vec M = \vec 0 \in \mathbb{R}^{|\mathcal{S}|\times|\mathcal{S}|}$; \Comment{Initialise the SR matrix as zero matrix}
\State $\vec Q = \vec 0 \in \mathbb{R}^{|\mathcal{S}|\times|\mathcal{A}|}$;
\While{not $done$}
\State $\theta \sim \mathcal{U}(0, 1)$;
\If{$\theta < \epsilon$} \Comment{$\epsilon$-greedy policy}
    \State $a \sim \mathcal{U}(\mathcal{A})$;
\Else
    \State $a = \argmax_{a\in\mathcal{A}}Q[s, a]$;
\EndIf
\State $s', r, done = env.step(a)$;
\State $\vec M[s, :] = \vec M[s, :] + \eta\left(\mathds{1}(s) + \gamma_{\text{SR}}(1-done)\vec M[s', :] - \vec M[s, :]\right)$; \Comment{TD-learning of the SR}
\State $r = r + \beta (\vec M[s, s'] - ||\vec M[:, s']||_{1})$;\Comment{Constructing intrinsic reward}
\State $\theta' \sim \mathcal{U}(0, 1)$;
\If{$\theta' < \epsilon$}
    \State $a'\sim \mathcal{U}(\mathcal{A})$;
\Else
    \State $a' = \argmax_{a\in\mathcal{A}}Q[s', a]$;
\EndIf
\State $\vec Q[s, a] = \vec Q[s, a] + \alpha\left(r + \gamma(1-done)\vec Q[s', a'] - \vec Q[s, a]\right)$;
\State $s = s'$;
\EndWhile
\end{algorithmic}
\end{algorithm}

\subsection{Evaluations given the fixed SR.} Conforming to our analysis of $r_{\text{SR-R}}$ with fixed SR (Section~\ref{sec: method}), we additionally evaluate SARSA-SR/FR/SRR with the corresponding intrinsic rewards constructed based on fixed SR/FR matrix on RiverSwim and SixArms (Table~\ref{tab: tabular_tasks_fixed}. Similar to what we found in the grid worlds (Figure~\ref{fig: grids_coverage_full_sr}), both SARSA-SR and SARSA-FR perform worse than their online-SR counterparts (note one exception being SARSA-FR on SixArms). However, in contrast to the decrease in exploration efficiency of SARSA-SRR in grid worlds, we found that fixing the SR actually improves the performance of SARSA-SRR. Hence, in accord with our analysis in Section~\ref{sec: method}, the cause for the improved empirical performance of $r_{\text{SR-R}}$ does not lie solely in the online learning process of SR, but might stems from the inherent ``bottleneck-seeking'' property of $r_{\text{SR-R}}$.

{
\begin{table}
\caption{Evaluations on RiverSwim and SixArms with intrinsic rewards based on fixed SR/FR (averaged over 100 seeds,
    numbers in the parentheses represents standard errors).}
  \renewcommand{\arraystretch}{1.2}
  % \hspace{-40pt}
  \centering
  \begin{tabular}{cccc} 
    \hline
        & SARSA-SR & SARSA-FR & SARSA-SRR \\ [0.5ex] 
        \hline\hline
        RiverSwim & 327,402  & 278,096 & \textbf{3,096,913}\\ 
           & (787,118) & (666,752) & (230,059)\\
        SixArms & 969,781  & 1,143,037 & \textbf{2,059,424}\\
        & (2,895,306) & (1,939,021) & (3,292,936) \\  
        \hline
  \end{tabular}
  \vspace{1ex}
  \label{tab: tabular_tasks_fixed}
  % \vspace{-15pt}
\end{table}
}

\subsection{Ablation studies of SPIE in discrete tasks}
\label{sec: sarsa_srr_ablation}

We perform ablation studies on SARSA-SRR for further demonstration of the utility of the SPIE objective of combining both the prospective and retrospective information. We firstly show that prospective information alone cannot yield strong exploration, whereas utilising solely the retrospective information maintains the strong explorative performance. We consider two variants of SARSA-SRR, SARSA-SRR(a) and SARSA-SRR(b), with the respective intrinsic rewards as following.
\begin{equation}
    \mathcal{R}_{\text{SR-R(a)}}(s, a, s')= \hat{M}[s, s']\,, \quad \mathcal{R}_{\text{SR-R(b)}}(s, a, s') = -||\hat{M}[:, s']||_{1}\,,
\end{equation}

From Table~\ref{tab: tabular_tasks_ablation}, we observe that utilising the prospective information alone for exploration yields suboptimal performance, hence empirically justifying the utility of the SPIE framework. However, we do observe that utilising the retrospective information alone yields near- or supra-optimal performance. Together, the results indicate that the global topological information contained in the retrospective information is essential for intrinsic exploration purposes. 

We argue that the dynamic balancing between exploring states with high uncertainty and bottleneck states is a key factor driving the empirical success of SPIE. In order to test this hypothesis, we devise a variant of the $\mathcal{R}_{\text{SR-R}}$.
\begin{equation}
    \mathcal{R}_{\text{SR-R(c)}} = ||\hat{M}[s, :]||_{1} - \hat{M}[s, s']\,,
\end{equation}
Intuitively, $\mathcal{R}_{\text{SR-R(c)}}$ provides an intrinsic motivation for taking transitions that lead to states that are less reachable from $s$, which only yields exploration towards states of high uncertainty, but does not provide any motivation towards bottleneck states. Indeed, as we observe from Table~\ref{tab: tabular_tasks_ablation} that SARSA-SRR(c) also yields suboptimal performance, providing empirical evidence supporting the benefits of SPIE in driving the agents towards bottleneck states. 

{
\begin{table}\small
\caption{Ablation studies of SARSA-SRR on RiverSwim and SixArms.}
  \renewcommand{\arraystretch}{1.2}
  % \hspace{-40pt}
  \centering
  \begin{tabular}{ccccc} 
    \hline
         & SARSA-SRR & SARSA-SRR(a) & SARSA-SRR(b) & SARSA-SRR(c)\\ [0.5ex] 
        \hline\hline
        RiverSwim & $\mathbf{2,547,156 }$ & 127,703  & $\mathbf{2,629,947}$ & 95,691\\ 
           & (479,655) & (530,564) & (930,170) & (181,216)\\
        SixArms & $\mathbf{2,199,291}$ & 893,530  & $\mathbf{1,902,553}$ & 562,346 \\
        & (1,024,726) & (2,601,324) & (2,211,960) & (1,748,455) \\  
        \hline
  \end{tabular}
  \vspace{1ex}
  \label{tab: tabular_tasks_ablation}
  % \vspace{-15pt}
\end{table}
}

\section{Further results on exploration in grid worlds}
\subsection{Transient dynamics of exploration.} 
We look more closely at the transient dynamics of the considered agents during pure exploration in \textit{Cluster-simple-large} (where \textit{Cluster-simple-large} denotes the $20\times 20$ grid world with two clusters). We observe that in the absence of external reinforcement, SARSA-SR, regardless of based on intrinsic rewards given either online-learned or fixed SR matrix, exhibits minimal exploration (Figure~\ref{fig: sr_20x20_full}~\ref{fig: sr_20x20_online}). This is largely due to its local exploration behaviour. For SARSA-FR, we observe significant difference between using online-trained and fixed FR matrix, where exploration with intrinsic rewards based on fixed FR completed disrupts exploration, only exploring a small proportion of the environment. In contrast, we observe that SARSA-SRR consistently fully explores both clusters (repeatedly) under both conditions. Additionally, by closely examining the transient dynamics during the exploration phase, we observe the ``cycling'' behaviour\footnote{see the attached videos in supplementary materials for the full exploration dynamics for the considered agents}.

\begin{figure*}
    \centering
    \begin{subfigure}[b]{\linewidth}
        \centering
        \includegraphics[width=\textwidth]{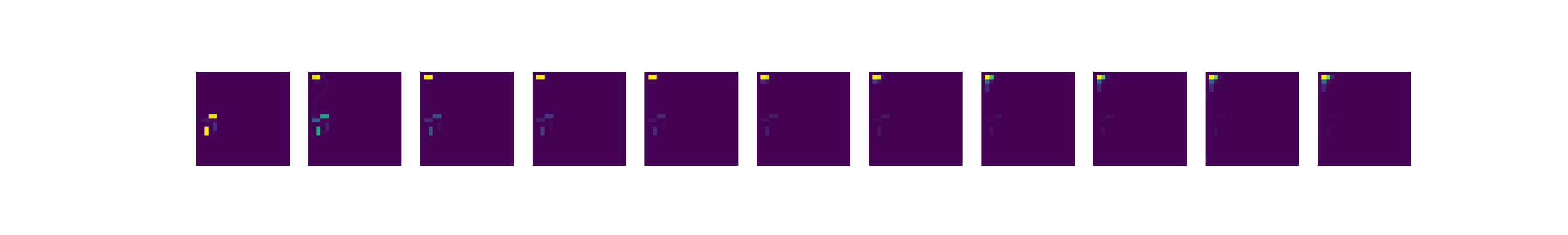}
        \caption{}
        \label{fig: sr_20x20_full}
    \end{subfigure}
    \hfill
    \begin{subfigure}[b]{\linewidth}
        \centering
        \includegraphics[width=\textwidth]{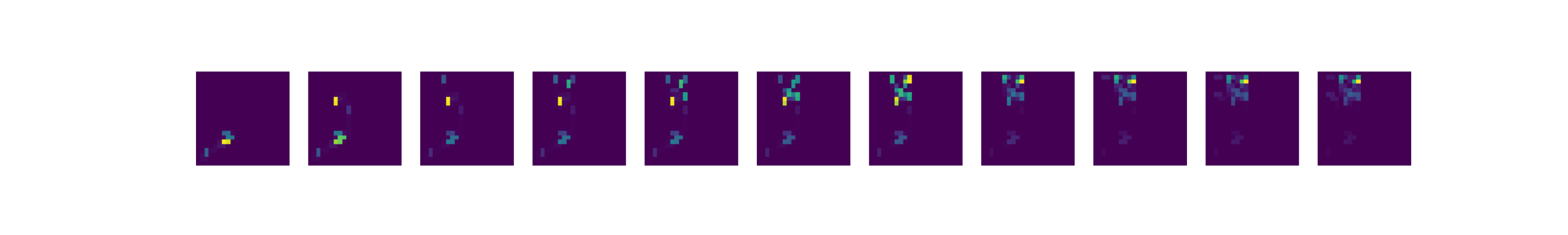}
        \caption{}
        \label{fig: sr_20x20_online}
    \end{subfigure}
    \hfill
    \begin{subfigure}[b]{\linewidth}
        \centering
        \includegraphics[width=\textwidth]{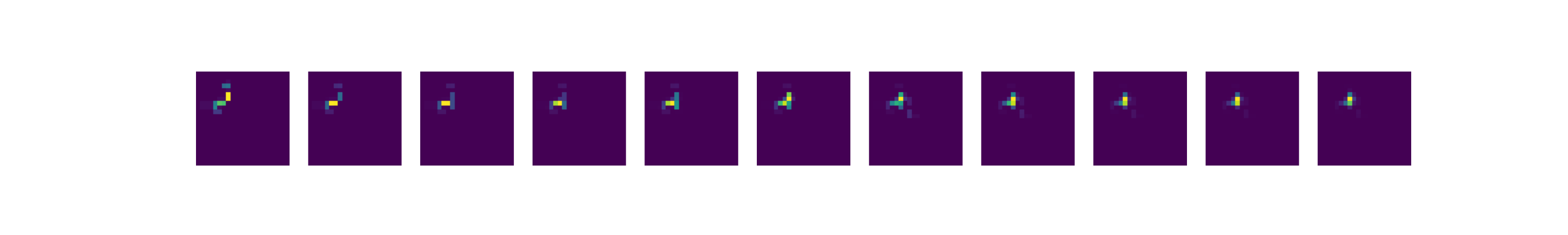}
        \caption{}
        \label{fig: fr_20x20_full}
    \end{subfigure}
    \hfill
    \begin{subfigure}[b]{\linewidth}
        \centering
        \includegraphics[width=\textwidth]{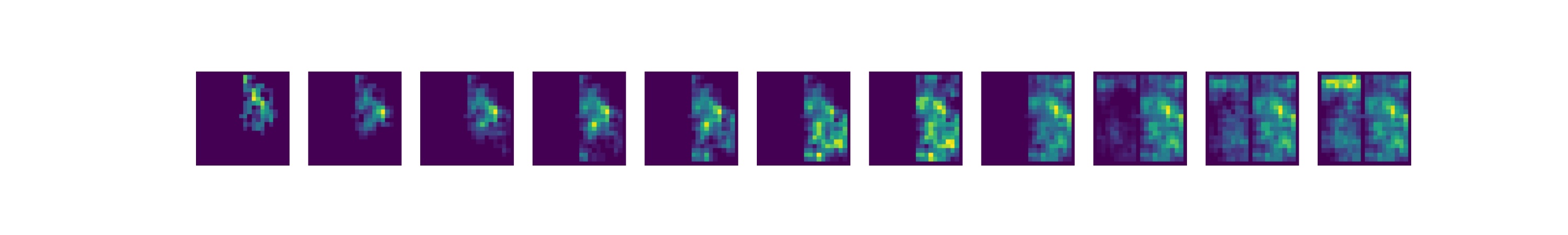}
        \caption{}
        \label{fig: fr_20x20_online}
    \end{subfigure}
    \hfill
    \begin{subfigure}[b]{\linewidth}
        \centering
        \includegraphics[width=\textwidth]{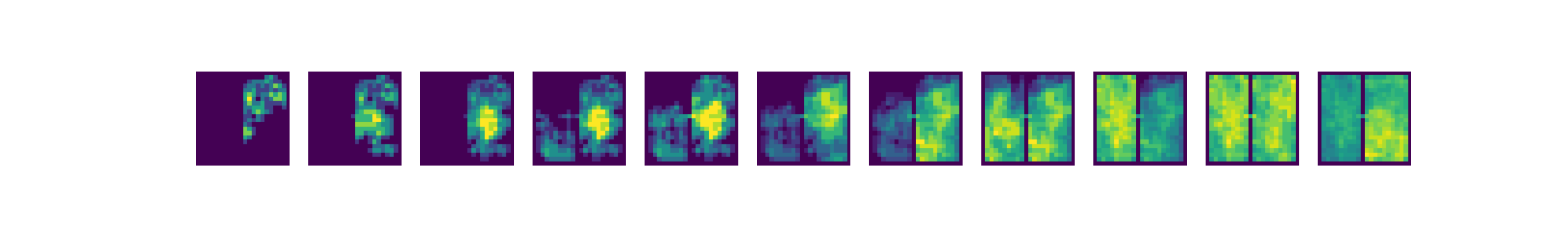}
        \caption{}
        \label{fig: sra_20x20_full}
    \end{subfigure}
    \hfill
    \begin{subfigure}[b]{\linewidth}
        \centering
        \includegraphics[width=\textwidth]{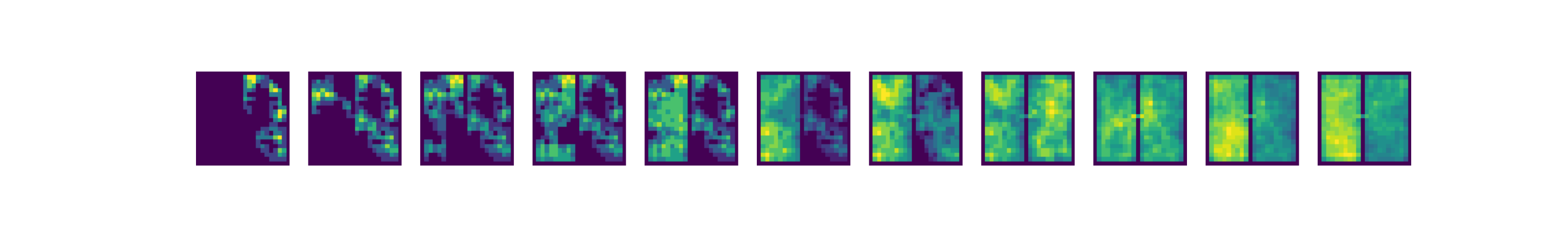}
        \caption{}
        \label{fig: sra_20x20_online}
    \end{subfigure}
   %  \vspace{-5pt}
       \caption{\textbf{Pure exploration given fixed SR / FR measures.} Temporal
       evoluation of state coverage heatmaps over $6000$ training steps of (a)
       SRASA-SR; (c) SARSA-FR; (e) SARSA-SRA agents with intrinsic rewards based
       on fixed SR/FR measures in \textit{OF-small}; and (b), (d), (f) for the counterparts with
       online-trained SR/FR measures in the $20\times 20$ \textit{Cluster-simple}
       grid world. From left to right: $200, 400, 600, 800, 1000, 1500, 2000, 3000, 4000, 5000, 6000$ steps.}
       \label{fig: grid_20x20_full}
       % \vspace{-10pt}
\end{figure*}

\subsection{Effect of optimistic initialisation.} 
We note that across all considered SARSA agents, the Q values were initialised to be $0$ for all state action pairs. Given that all SR entries are non-negative, we know that $r_{\text{SR-R}}$ only admits negative rewards, hence the zero-initialisation yields optimistic initialisation, which encourages the agent to explore~\citep{brafman2002r, strehl2008analysis}. To disentangle the effect of SPIE from optimistic initialisation, we perform the ablation study on pure exploration with augmented SARSA-SR and SARSA-FR agents with optimistic initialisation. Specifically, we note that the maximum value the SR entries can take is $\frac{1}{1-\gamma}$, and additionally since the FR entries, by definition, are always less than or equal to the corresponding SR entries, we initialise the Q values for all state-action pairs for both SARSA-SR and SARSA-FR to be $\frac{1}{1-\gamma}$. We evaluate the exploration efficiency for the optimistically augmented agents on the grid worlds (Figure~\ref{fig: grid_exploration_optimistic}), and we observe that despite the optimistic initialisation improves the performance of both SARSA-SR and SARSA-FR relative to their corresponding naive counterparts, the performance differences in terms of exploration efficiency between the augmented agents and SARSA-SRR are significant, hence justifying the utility of the SPIE framework independent of the optimistic initialisation.

\begin{figure}
    \centering
    \includegraphics[width=\linewidth]{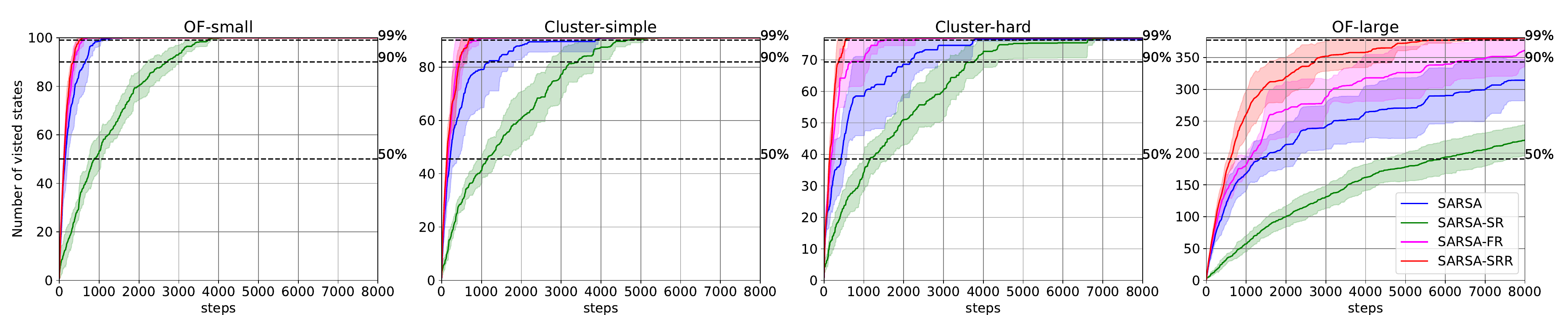}
    \caption{\textbf{Ablation study on optimistic initialisation on exploration efficiency.} We evaluate SARSA, SARSA-SRR, and optimistically augmented SARSA-SR and SARSA-FR on the considered grid worlds (Figure~\ref{fig: grids_demo}).}
    \label{fig: grid_exploration_optimistic}
\end{figure}

\section{Further results on deep RL implementation of SPIE in Atari games}

\subsection{Ablation study on the effect of predictive reconstruction auxiliary task}

In our implementation of DQN-SF-PF, by following relevant literature~\cite{oh2015action, machado2020count}, we include an additional sub-module in the neural architecture for predicting action-dependent future observation, which is trained via minimising the predictive reconstruction error. The purpose of including this sub-module is purely for learning better latent representations underlying the visual observation. We validate the utility of such predictive reconstruction auxiliary supervision by performing ablation study. We implemented an alternative version of DQN-SF-PF, removing the visual reconstruction sub-module, and test on Montezuma's Revenge. The resulting model achieves $551.5$ points (averaged over $5$ random seeds, s.e. equals $618.4$). We observe that there is a significant decrease from standard DQN-SF-PF (Table~\ref{tab: atari}), indicating the importance of stronger representation learning given the predictive reconstruction auxiliary task. Moreover, given the reported performance of $398.5$ points (s.e., equals $230.1$) of DQN-SF in the absence of predictive reconstruction auxiliary task from~\citet{machado2020count}, we observe that the SPIE objective still yields improved performance over exploration with SF alone, justifying the utility of SPIE irrespective of the specific neural architecture we choose. 

\section{Experiment Details}
Here we provide further details of the experiments presented in the main paper.

\textbf{Tabular tasks.} We run hyperparameter sweeps for all considered agents (SARSA, SARSA-SR, SARSA-FR, SARSA-SRR) on the following hyperparameters: $\{0.005, 0.05, 0.1, 0.25, 0.5\}$ for learning rate of TD learning for the Q values ($\alpha$); $\{0.005, 0.05, 0.1, 0.25, 0.5\}$ for learning rate of TD learning for the SR/FR matrices ($\eta$); $\{0.5, 0.8, 0.9, 0.95, 0.99\}$ for the discounting factor defining the SR/FR formulation ($\gamma_{\text{SR/FR}}$); $\{1, 10, 50, 100, 1000, 10000\}$ for the multiplicative scaling factor controlling the scale of the intrinsic rewards ($\beta$); $\{0.01, 0.05, 0.1\}$ for the degree of randomness in $\epsilon$-greedy exploration ($\epsilon$). The complete sets of optimal hyperparameters for the reported performance of the considered agents in Table~\ref{tab: tabular_tasks} (and for the corresponding agents with intrinsic rewards based on fixed SR/FR matrix; Table~\ref{tab: tabular_tasks_fixed}) in shown in Table~\ref{tab: hyperparam_tabular}.

{
\begin{table}
\caption{Hyperparameters for the considered agents in the tabular hard-exploration tasks (the values in parentheses are the corresponding hyperparameter values for the learning of the PR).}
  \renewcommand{\arraystretch}{1.2}
  % \hspace{-40pt}
  \centering
  \begin{tabular}{cccccccc} 
    \hline
        & agent & $\alpha$ & $\eta$ & $\gamma$ & $\gamma_{\text{SR/FR}}$ & $\beta$ & $\epsilon$ \\ [0.5ex] 
        \hline\hline
        \textbf{RiverSwim} & SARSA & $0.005$ & - & $0.95$ & - & - & $0.01$\\
                           & SARSA-SR & $0.25$ & $0.1$ & $0.95$ & $0.95$ & $ 10$ & $0.1$ \\
                           & SARSA-FR & $0.25$ & $0.01$ & $0.95$ & $0.95$ & $50$ & $0.1$ \\
                           & SARSA-SRR & $0.1$ & $0.25$ & $0.95$ & $0.95$ & $10$ & $0.01$\\
                           & SARSA-SR-PR & $0.25$ & $0.25 (0.1)$ & $0.95$ & $0.95 (0.99)$ & $1$ & $0.01$\\
                           & SARSA-SR (fixed) & $0.01$ & - & $0.95$ & $0.95$ & $10$ & $0.05$ \\
                           & SARSA-FR (fixed) & $0.1$ & - & $0.95$ & $0.95$ & $10$ & $0.1$ \\
                           & SARSA-SRR (fixed) & $0.25$ & - & $0.95$ & $0.95$ & $10$ & $0.01$ \\
        \hline
        \textbf{SixArms} & SARSA & $0.5$ & - & $0.95$ & - & - & $0.01$\\
                         & SARSA-SR & $0.1$ & $0.01$ & $0.95$ & $0.99$ & $100$ & $0.01$ \\
                         & SARSA-FR & $0.1$ & $0.01$ & $0.95$ & $0.99$ & $100$ & $0.01$ \\
                         & SARSA-SRR & $0.01$ & $0.01$ & $0.95$ & $0.99$ & $10000$ & $0.01$ \\ 
                         & SARSA-SR-PR &  $0.05$ & $0.25 (0.25)$ & $0.95$ & $0.95 (0.99)$ & $10$ & $0.01$ \\
                         & SARSA-SR (fixed) & $0.5$ & - & $0.95$ & $0.95$ & $1$ & $0.01$ \\
                         & SARSA-FR (fixed) & $0.5$ & - & $0.95$ & $0.95$ & $1$ & $0.01$ \\
                         & SARSA-SRR (fixed) & $0.5$ & - & $0.95$ & $0.95$ & $10$ & $0.01$ \\
        \hline
  \end{tabular}
  \vspace{1ex}
  \label{tab: hyperparam_tabular}
  % \vspace{-15pt}
\end{table}
}

\textbf{Exploration in grid worlds.} For all presented results in the grid worlds, we use the hyperparameters $(0.1, 0.1, 0.95, 0.95, 1.0, 0.1)$ for $(\alpha, \eta, \gamma, \gamma_{\text{SR/FR}}, \beta, \epsilon)$.

\textbf{MountainCar experiment.} We use the $128$-dimensional random Fourier features, defined over the two-dimensional state space (location$\times$speed), as the state representation. We use the hyperparameters $(0.1, 0.2, 0.2, 0.99, 0.95, 0.95, 1000, 0.3)$ for $(\alpha, \eta, \eta_{\text{PR}}, \gamma, \gamma_{\text{SR}}, \gamma_{\text{PR}}, \beta, \epsilon)$, where $\eta_{\text{PR}}$ and $\gamma_{\text{PR}}$ are the learning rate and discounting factor values for the PR, respectively.

\textbf{Atari experiments.} 
The neural architecture of the deep RL implementation shown in Figure~\ref{fig: dqn_demo}, here we provide the specific hyperparameters of the architecture. The \textit{Conv} block is a convolutional network with the configuration $(4, 84, 84, 0, 2) - ReLU - (64, 40, 40, 2, 2) - ReLU - (64, 6, 6, 2, 2) - ReLU - (64, 10, 10, 0, 0) - FC(1024)$, where the tuple represents a 2-dimensional convolutional layer with the architecture (num\_filters, kernel\_width, kernel\_height, padding\_size, stride), and $FC(1024)$ represents a fully connected layer with $1024$ hidden units. We take the output of the \textit{Conv} block as the $1024$-dimensional state representation given the observation, which is then subsequently used for computing the SF and the PF. The action input is transformed into a high-dimensional embedding through a linear transformation, $FC(2048)$. The MLP for the predictive reconstruction block is $FC(2048)-ReLU$, for the Q-value estimation block is $FC(|\mathcal{A}|)$, for the SF head block is $FC(2048)-ReLU-FC(1024)$, for the PF head block is $FC(2048)-ReLU-FC(1024)$. The \textit{Deconv} block is $FC(2048) - FC(1024) - ReLU - FC(6400) - Reshape((64, 10, 10)) - \langle 64, 6, 6, 2, 2\rangle - \langle 64, 6, 6, 2, 2\rangle - \langle 1, 6, 6, 0, 2\rangle - Flatten$, where the tuple represents a 2-dimensional deconvolutional layer with parameters $\langle$ num\_filters, kernel\_width, kernel\_height, padding\_size, stride $\rangle$.

% \begin{figure}
%     \centering
%     \includegraphics[width=\linewidth]{figures/optimal_policy_varying_gamma.pdf}
%     \caption{Optimal value function and policy of varying $\gamma$.}
%     \label{fig: VI_varying_gamma}
% \end{figure}

\end{document}